\documentclass[twoside]{article}

\usepackage[accepted]{aistats2022}

\usepackage[utf8]{inputenc} % allow utf-8 input
\usepackage[T1]{fontenc}    % use 8-bit T1 fonts
\usepackage{hyperref}       % hyperlinks
\usepackage{url}            % simple URL typesetting
\usepackage{booktabs}       % professional-quality tables
\usepackage{amsfonts}       % blackboard math symbols
\usepackage{nicefrac}       % compact symbols for 1/2, etc.
\usepackage{microtype}      % microtypography
\usepackage{xcolor}         % colors
\usepackage{microtype}
\usepackage{graphicx}
\usepackage{booktabs} % for professional tables

\usepackage{caption}
\usepackage{subcaption}
\usepackage{amssymb,amsmath,amsthm}
\usepackage{epstopdf}
\usepackage{algorithm,algcompatible}
\usepackage{mathtools}
\usepackage{wrapfig}
\usepackage{soul}
\usepackage{tikz}
\usetikzlibrary{fit,positioning,arrows,automata}
\usepackage[round]{natbib}

\mathtoolsset{showonlyrefs}

\newtheorem{theorem}{Theorem}

\newtheorem{lemma}{Lemma}

\newtheorem{remark}{Remark}
\newtheorem{example}{Example}
\newtheorem{defn}{Definition}
\def\tT{{\text{T}}}
\def\KL{{\text{KL}}}
\def\bx{{\bf x}}
\def\bC{{\bf C}}
\def\bK{{\bf K}}
\def\bM{{\bf M}}

\def\bB{{\bf B}}
\def\bY{{\bf Y}}
\def\E{{\mathbb{E}}}
\def\P{{\mathbb{P}}}
\def\mF{{\mathcal{F}}}
\def\mO{{\mathcal{O}}}
\def\<{{\langle}}
\def\>{{\rangle}}

\def\R{{\mathbb{R}}}
\def\holder{{H$\ddot{\text{o}}$lder's }}

\newcommand{\ones}{\mathbf{1}}

\newcommand{\jiao}[1]{{\color{black}{#1}}}

\newcommand{\PiGamma}{{\Pi_\Gamma^m((\mu_k)_{k\in \Gamma})}}
\newcommand{\RC}{{R_C^\Gamma}}

\DeclarePairedDelimiter\ceil{\lceil}{\rceil}

\begin{document}

% \maketitle

\twocolumn[

\aistatstitle{On the complexity of the optimal transport problem with graph-structured cost}

\aistatsauthor{ Jiaojiao Fan$^*$ \And Isabel Haasler$^*$ \And  Johan Karlsson\And  Yongxin Chen }

\aistatsaddress{ Georgia Tech \And  KTH  \And KTH  \And Georgia Tech } ]

\begin{abstract}
    Multi-marginal optimal transport (MOT) is a generalization of optimal transport to multiple marginals. Optimal transport has evolved into an important tool in many machine learning applications, and its multi-marginal extension opens up for addressing new challenges in the field of machine learning. However, the usage of MOT has been largely impeded by its computational complexity which scales exponentially in the number of marginals. Fortunately, in many applications, such as barycenter or interpolation problems, the cost function adheres to structures, which has recently been exploited for developing efficient computational methods. In this work we derive computational bounds for these methods. 
    In particular, with $m$ marginal distributions supported on $n$ points, we provide a $ \mathcal{\tilde O}(d(\mathcal{T})m n^{w(G)+1}\epsilon^{-2})$ bound for a $\epsilon$-accuracy when the problem is associated with a graph that can be factored as a junction tree with diameter $d(\mathcal{T})$ and tree-width $w(G)$.
    For the special case of the Wasserstein barycenter problem, which corresponds to a star-shaped tree, our bound is in alignment with the existing complexity bound for it.
    %In particular, with $m$ marginal distributions supported on $n$ points, we provide a $ \mathcal{\tilde O}(d(G)m n^2\epsilon^{-2})$ bound for a $\epsilon$-accuracy when the problem is associated with a tree width diameter $d(G)$.
        % \ih{ [to what? Maybe add $d(G)$? Otherwise I think in the other papers they hide the same constants in the abstract and introduction.]}
        %This bounds depends linearly on the number of known marginals. 
\end{abstract}

% \jiao{TODO list

% - experiments: general graph, recheck stable issue

% - rewrite our contribution 
% (include other final replies)

% - Put general graph in a higher position

% - add one table in the main paper

% - compress theorem 1 proof
% }

\section{Introduction}
\label{sec:introduction}

 The history of optimal transport can be traced back to the 18-th century when the French mathematician Monge introduced this tool for his engineering projects. In optimal transport problems one seeks an optimal strategy to move resources from an initial distribution to a target one.
 This theory has initially %The optimal transport theory has historically
 had a tremendous impact to fields such as economics and logistics. During the last decades, with new efficient computational methods \citep{villani2009optimal,cuturi2013sinkhorn} and more available computational power, optimal transport theory has also been used for addressing a broad class of problems both within the machine learning community \citep{peyre2019computational, solomon2014wasserstein, solomon2015convolutional, arjovsky2017wasserstein}, but also in related fields such as imaging \citep{haker2004optimal} and systems and control \citep{chen2016relation}.

 Multi-marginal optimal transport (MOT) is a natural extension of standard optimal transport to scenarios with more than two marginal distributions. In the discrete setting, the objective of MOT is to find an optimal coupling between $m$ marginals $\mu_1,\dots,\mu_m\in \mathbb{R}_+^n$ over $X$, where $X$ is a discrete space with support in $n$ points.
A $m$-mode tensor $\bB \in \mathbb{R}_+^{n^m}$ is a feasible transport plan if it satisfies the assigned marginals,
$P_k(\bB) = \mu_k$, where 
\begin{equation} \label{eq:proj_bruteforce}
    [ P_k(\bB) ] (x_k) = \sum_{ \bx \setminus x_k} \bB( \bx), \quad \text{for all } x_k \in X,
\end{equation}
where $\bx = (x_1,\dots, x_m) \in X^m$.
%The set of feasible couplings between the marginals $\mu_1,\dots, \mu_m$ is denoted by $\Pi^m(\mu_1,\dots,\mu_m)$.
%More generally, 
In this paper we consider a version of this problem where the marginals are typically only imposed on a subset of the transport tensors nodes, and we denote this subset of indices by $\Gamma \subset \{1,\dots,m\} $. The set of feasible transport plans consistent with these marginals $\{\mu_k\}_{k\in\Gamma}$ is then %in the set
% {\red we may need to find a way to emphasize the difference between the number of marginals $m$ and the number of given marginals $|\Gamma|$}
\begin{equation}
    \PiGamma = \{ \bB \in \mathbb{R}^{n^m} : P_k(\bB) = \mu_k, \forall k \in \Gamma  \}.
\end{equation}
%\begin{equation}
%    \Pi^m(\mu_k, k\in \Gamma) = \{ \bB \in \mathbb{R}^{n^m} : P_k(\bB) = \mu_k \ \forall k \in \Gamma  \}.
%\end{equation}
Given a non-negative cost tensor $\bC \in \mathbb{R}^{n^m}_+$, where $\bC(\bx)$ denotes the cost associated with a unit mass on the tuple $\bx$, the multi-marginal optimal transport problem reads
\begin{align}\label{eq:unregu mot}
    \min_{ \bB \in \PiGamma } \langle \bC, \bB \rangle.
\end{align}

%\ih{ \bf [ Introduce $\epsilon$-approximation ]}
The MOT problem is a linear program, thus, in principle, the simplex algorithm can be used to solve it exactly. The complexity however explodes quickly as the problem size increases. In practice, the MOT is solved approximately instead. The goal of these approximation algorithms is to find $ \widehat \bB \in \PiGamma$ such that $\langle \bC, \widehat \bB \rangle$ is an $\epsilon$-approximation of the MOT problem \eqref{eq:unregu mot}. That is, $\widehat \bB$ is an approximation of the transport tensor and satisfies
\begin{equation}
    \langle \bC, \widehat \bB \rangle \leq  \min_{ \bB \in \PiGamma} \langle \bC, \bB \rangle + \epsilon.
\end{equation}
A popular method to approximately solve the MOT problem \eqref{eq:unregu mot} is to solve an entropic regularized version of it where an entropy barrier term is added to the objective. This regularized problem can be solved by the renowned Sinkhorn iterations \citep{DemSte40,cuturi2013sinkhorn}.

%\ih{ \bf [ Talk about some known bounds ]}
\paragraph{Related work:}
A fundamental question in the study of MOT algorithms is understanding their complexities, and several complexity bounds have been derived over the last few years for various MOT algorithms \citep{lin2019complexity, altschuler2020polynomial, carlier2021linear}. The best known complexity bound for the general multi-marginal Sinkhorn iterations is%
% \footnote{Here $\mathcal{ \tilde O}$ denotes the growth rate, it ignores polylogarithmic factors in $n$. See "Notation" for the definition.}
$\mathcal{ \tilde O}(\frac{m^3n^m}{\epsilon^2})$
\citep{lin2019complexity} \jiao{with greedy updates}, which scales exponentially in the number of marginals $m$. This is not surprising as the size of the variable $\bB$ grows exponentially. 
This complexity bound can be improved by exploiting the structure of the cost tensor $\bC$. A well-known example is the Wasserstein barycenter problem where the cost can be decomposed into pairwise costs between the marginals and the barycenter. \citet{kroshnin2019complexity} shows that the iterative scaling algorithm finds an $\epsilon$-approximate solution to the barycenter between $L$ distributions in $ \mathcal{ \tilde O}(\frac{ L n^2}{\epsilon^2})$ operations. A more general class of costs where better computation complexity can be achieved is associated with the tree structure (see Section \ref{sec:graph}).
Such structures appear in various applications, such as barycenter problems \citep{lin2020fixedsupport, kroshnin2019complexity}, interpolation problems \citep{solomon2015convolutional}, and estimation problems \citep{elvander2020multi}.
 It was shown in \citet{haasler2021pgm} that a complexity bound for MOT problems with tree-structured cost (including the barycenter problem as a special case) is  $ \mathcal{ \tilde O}(\frac{ m^4 n^2}{\epsilon^2})$, where $m$ denotes the number of marginals.
 Many other MOT problems are structured according to graphs that contain cycles, e.g., in the generalized Euler flow problem \citep{benamou2015iterative}, control applications \citep{haasler2020optimal}, and multi-species problems \citep{haasler2021scalable}.
In \citet{altschuler2020polynomial}, it was shown that the complexity for MOT with general graph-structured cost scales polynomially as the number of marginals increases, as long as the tree-width of the graph is properly bounded, but they do not provide explicit dependencies on the parameters.
Note that some other structures of the cost tensors such as the low rank property can be leveraged \citep{altschuler2020polynomial}, but these are very different to the graphical structure considered in this work.
% \jiao{But they leave the explicit dependencies of the parameters remain an open question.}
%  this may be reduced to $ \mathcal{ \tilde O}(\frac{m^4 n^2 }{\epsilon^2})$.

\paragraph{Our contribution:}
The purpose of this work is to provide a tighter complexity bound for solving the MOT problem with general graph-structured costs.
Tree-structured optimal transport problems are often formulated as a sum of bi-marginal optimal transport problem.
The numerical scheme is
often based on regularizing each of the bi-marginal problems locally
in previous work.
However, if the underlying graph structure contains cycles, there is no such representation of the problem.
In this work we suggest to use a regularization on the 
% multi-marginal transport 
MOT
tensor similar to the one suggested in \citet{carlier2021linear}.
This regularization also simplifies the complexity analysis (see Remark \ref{rem:regularizer}). 
For the cases where the MOT problem is structured according to a tree, i.e., the graph does not contain any cycles, we show that an $\epsilon$-approximation of the solution can be found within
$ \mathcal{\tilde O}(\bar d(G) m n^2\epsilon^{-2})$ operations, where $\bar d(G)$ denotes the average distance between two leaves of the tree.
This improves on the previous result $\mathcal{\tilde O}(m^4 n^2\epsilon^{-2})$ for tree-structured MOT in \cite{haasler2021pgm}.
% Moreover,
Especially
 for the barycenter problem, which corresponds to the special case of a star-shaped graph, this matches the best known bound when no further acceleration of the method is applied.
The framework in this paper also treats a class of MOT problems that is much larger than what can be described by bi-marginal OT problems.
In the case of a general graph $G$, the complexity is $ \mathcal{ \tilde O}  \left( \bar d(\mathcal{T}) m  n^{w(G)+1} \epsilon^{-2} \right)$, where $\mathcal{T}$ is a minimal junction tree over the graph $G$, and $w(G)$ is the tree-width of $G$.
The best-known complexity bounds for optimal transport without acceleration are summarized in Table~\ref{tab:bounds}.
%
%\begin{remark}
There are accelerated versions of the Sinkhorn algorithm, see, e.g., \citet{lin2019complexity,kroshnin2019complexity} that can improve the dependence with respect to $\epsilon$ from $\epsilon^{-2}$ to $\epsilon^{-1}$. Note that these accelerations cannot improve the dependence over $m$ or $n$. Since the algorithm studied in this work is not accelerated, we compare the complexity bounds only to algorithms with no acceleration.
%\end{remark}
%
\begin{table*}[tb]
\begin{minipage}{0.72\linewidth}
	\caption{Best-known complexity bounds for optimal transport without acceleration 
% 	\jiao{add stoc/deter bounds}
	}
	\label{tab:bounds}
	\begin{center}
		\begin{small}
\begin{tabular}{lll}
\toprule
Problem & Complexity & Paper \\
\toprule
Bi-marginal optimal transport   &$\mathcal{\tilde O}(n^2\epsilon^{-2})$ & \citet{dvurechensky2018computational} \\
\midrule
Barycenter optimal transport & $\mathcal{\tilde O}(m n^2\epsilon^{-2})$ & \citet{kroshnin2019complexity} \\
\midrule
General MOT   & $\mathcal{\tilde O}(m^3 n^m\epsilon^{-2})$ & \citet{lin2019complexity} \\
\midrule
Tree-structured MOT  &  $ \mathcal{ \tilde O}( m^4 n^2\epsilon^{-2})$ & \citet{haasler2021pgm} \\
\midrule
Tree-structured MOT  & $\mathcal{\tilde O}(\bar d(G)m n^2\epsilon^{-2})$ & Ours  \\
\midrule
balanced Graph-structured MOT  & $ \mathcal{ \tilde O}  ( \bar d(\mathcal{T}) m  n^{w(G)+1} \epsilon^{-2} )$ & Ours\\ %\footnote{Here, \mathcal{T} denotes a junction tree over the graph $G$}  \\
\bottomrule
\end{tabular}
\end{small}	\end{center}
\end{minipage} \hfill
\begin{minipage}{0.2\linewidth}
%\begin{figure}
\centering
%\vspace{-20pt}
    \begin{tikzpicture}
\def \radius {35pt}
\tikzstyle{main}=[circle, minimum size = 20pt, thick, draw =black!80]
\node[main] at (360:0mm) (center) {\small $\mu$};
\foreach \i  in {1,...,5}{
  \node[main,,fill=black!10] at ({\i*70}:\radius) (u\i) { \small $\mu_{\i}$};
  \draw[-, thick] (center)--  (u\i);
}
\end{tikzpicture}

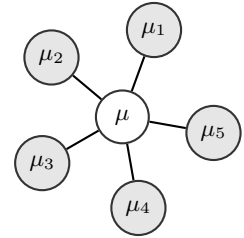
\captionof{figure}{Graph associated with a barycenter problem \eqref{eq:barycenter_pairwise}, where $L=5$.}
%    \caption{Graph associated with a barycenter problem \eqref{eq:barycenter_pairwise}, where $L=5$. }
    \label{fig:barycenter_graph}
%\end{figure}
\end{minipage}
\end{table*}

\paragraph{Notation:}

For a matrix $C \in \mathbb{R}^{n \times n}$, we denote $\|C\|_\infty$ its largest element.
We denote a graph as the tuple $G=(V,E)$, where $V$ is the set of vertices, and $E$ is the set of edges. For a vertex $k\in V$, we denote the set of neighbouring vertices by $N(k) \subset V$. Let $\ones_d$ denote the all-ones vector/matrix/tensor in $\R^d$, and let $\exp(\cdot), \log(\cdot), \odot$, and $./$ denote the element-wise exponential, logarithm, multiplication, and division of tensors, respectively.
% We use $p(m,n,\epsilon)=\mathcal{ O}(q(m,n,\epsilon))$ to denote that there exists a positive constant $c_1$ such that $p(m,n,\epsilon) \leq c_1 q(m,n,\epsilon)$.
\jiao{The $p(m,n,\epsilon)=\mathcal{\tilde O}(q(m,n,\epsilon))$ notation absorbs polylogarithmic factors related to $n$, i.e., there exist positive constants $c_2,c_3$ such that $p(m,n,\epsilon) \leq c_2 q(m,n,\epsilon)(\log n)^{c_3} $. }
% {\red do we need to use ${O}$?}

\section{Graph-structured MOT}\label{sec:graph}

In this paper we consider MOT problems with a cost that decouples according to a graph.
Such structures appear in many applications, for instance in barycenter problems \citep{lin2020fixedsupport, kroshnin2019complexity}, interpolation problems \citep{solomon2015convolutional}, and estimation problems \citep{elvander2020multi,SinHaaZha20}. In fact, one of the very first studies of MOT, on the generalized Euler-flow problem, has a graph-structured cost \citep{Bre89,benamou2015iterative}.
% \ih{[Cite this one or an older one?]} {\red both, done}.

%
%
% \ih{ \bf [ Example: Barycenter as special case ]}
\begin{example} \label{exa: barycenter}
(Fixed-support Wasserstein Barycenter).
    {\makeatletter
		\let\par\@@par
		\par\parshape0
	\everypar{}
    A special case of a graph-structured optimal transport problem is the fixed support barycenter problem with uniform weights
    % \begin{align} \label{eq:barycenter_pairwise}
    %   & \min_{\mu\in \mathbb{R}^n} \sum_{\ell=1}^{L} \frac{1}{L}  W(\mu, \mu_\ell), \\ \nonumber
    %     ~ \text{ with }~ & W(\mu,\nu)=\inf_{B \in \Pi(\mu,\nu)} \langle C, B \rangle,
    % \end{align}
        \begin{align} \label{eq:barycenter_pairwise}
        \min_{\mu\in \mathbb{R}^n} \sum_{\ell=1}^{L} \frac{1}{L}  W(\mu, \mu_\ell), 
        ~ \text{ with }~  W(\mu,\nu)=\inf_{B \in \Pi(\mu,\nu)} \langle C, B \rangle,
    \end{align}
    % \jiao{Can inf be replaced by min here?}
    % \ih{and $\sum_{k=1}^m \omega_k=1$ ? [Is this technically crucial?]} 
    where $\Pi(\mu,\nu)=\Pi^2_{\{1,2\}}(\mu,\nu)$ denotes the standard set of feasible transport plans for two marginals.
    The underlying structure can be described by a star-graph as illustrated in Figure~\ref{fig:barycenter_graph}.
    Problem \eqref{eq:barycenter_pairwise} can be written as the multi-marginal problem \eqref{eq:unregu mot}, where the cost tensor $\bC \in \mathbb{R}^{n^{L+1}}$ is defined as
    \begin{equation}\label{eq:barycentercost}
         \bC(x_1,\dots,x_L,x_{L+1}) = \sum_{\ell=1}^{L} \frac{1}{L} C(x_{L+1}, x_{\ell}),
    \end{equation}
    \par }
    and constraints are given on the set $\Gamma=\{1,\dots, \ell\}$.
\end{example}
% \begin{wrapfigure}{r}{0.3\textwidth}
% \centering
% %\vspace{-20pt}
%     \begin{tikzpicture}
% \def \radius {40pt}
% \tikzstyle{main}=[circle, minimum size = 20pt, thick, draw =black!80]
% \node[main] at (360:0mm) (center) {$\mu$};
% \foreach \i  in {1,...,5}{
%   \node[main,,fill=black!10] at ({\i*70}:\radius) (u\i) {$\mu_{\i}$};
%   \draw[-, thick] (center)--  (u\i);
% }
% \end{tikzpicture}
%     \caption{Graph associated with a barycenter problem \eqref{eq:barycenter_pairwise}, where $L=5$. }
%     \label{fig:barycenter_graph}
% \end{wrapfigure}
Similar to Example~\ref{exa: barycenter}, we can define a MOT problem that is structured according to any graph $G=(V,E)$.
%Let $G=(V,E)$ be a graph with vertices $V$ and edges $E$.
Therefore, we associate each vertex in $V$ with a marginal of the transport plan $\bB$, and each edge in $E$ with a pair-wise cost.
That is, for the interaction between vertices $k_1$ and $k_2$ we define a cost matrix $C^{(k_1,k_2)}$, and we let $E$ be the set of all these pair-wise interactions. Then the graph-structured cost tensor is defined by
% \footnote{\jiao{The cost tensors with graphical structure considered in this work are generally not of low-rank.} \ih{If we want to mention this, then maybe move it to the related work section?}{\red I agree. I have added a sentence in introduction.}}
\begin{equation} \label{eq:C_graph}
    \bC ( x_1,\dots,x_m) = \sum_{(k_1,k_2)\in E} C^{(k_1,k_2)}(x_{k_1},x_{k_2}).
\end{equation}
Problem \eqref{eq:unregu mot} with a cost tensor of the form \eqref{eq:C_graph} is called a graph-structured MOT problem \citep{haasler2020tree, haasler2021pgm}.

Many graph-structured optimal transport problems, for instance interpolation and barycenter problems, are naturally described by tree graphs, i.e., graphs that do not contain any cycles. 
Moreover, any graph can be converted into a tree using the junction tree technique \citep{KolFri09}, and we will use this representation to derive complexity bounds for general graph-structured MOT problems.
%Thus, for the ease of presentation, we focus on the cases of tree-structured MOT in the main paper and leave the discussion of general graphs to the supplementary material.
It should be noted that in the case of a tree-structured MOT problem we can without loss of generality consider the case, where $\Gamma$ is the set of leaves \cite[Proposition 3.4]{haasler2020tree}.

%\ih{[ In a tree each $k\in \Gamma$ has exactly one neighbour node. Can we find a good notation for this? Would be useful, e.g., in Algorithm \ref{algo:round} and Lemma \ref{lem:round bound}]}

% -------------------------------- Maybe remove the following: --------------------------------

% {\red what's the purpose of this argument?} \ih{I think the first sentence is somewhat relevant, since not only the barycenter problem can be expressed as MOT.} It is worth noting that any sum of bi-marginal optimal transport problems that decouple according to a tree can equivalently be expressed as one MOT problem. \st{However, in the case of general graphs MOT incorporates more information than the corresponding sum of bi-marginal OT problems.}

\section{Sinkhorn belief propagation algorithm}

In practical applications the MOT problem is often prohibitively large  for standard linear programming solvers, and one therefore has to resort to numerical methods to obtain an appropriate solution.
A well-known approach, based on the seminal work by \cite{cuturi2013sinkhorn}, is to regularize the objective in \eqref{eq:unregu mot} with an entropic barrier term \citep{benamou2015iterative}.
In particular, we introduce the barrier term
\begin{equation} \label{eq:regularization}
    H(\bB \mid \bM)=\langle\bB, \log(\bB)-\log(\bM) -\ones_{n^m}\rangle,
\end{equation}
where
\begin{equation}
    \bM(x_1,x_2,\ldots, x_m)=\prod_{k\in \Gamma} \mu_k(x_k).
\end{equation}
The regularized MOT problem reads then
\begin{align} \label{eq:ot_multi_reg}
    \min_{ \bB \in \PiGamma } \langle \bC, \bB \rangle + \eta H(\bB\mid \bM),
\end{align}
where $\eta>0$ is a small regularization parameter.

\begin{remark}\label{rem:regularizer}
%emphasize the differences between the above regularization \eqref{eq:regularization} and the standard one: i) Carlier's formulation ii) it is also different to the pairwise regularization \ih{ [Could you write this remark please? I'm not so sure what exactly to say]} 
Note that our choice of entropy regularizer is slightly different from the standard one $\langle\bB, \log(\bB)-\ones_{n^m}\rangle$ often used for the Sinkhorn algorithm. The extra term $-\langle\bB,\log(\bM)\rangle$ turns out to simplify the approximation procedure (there is no need to alter the marginal distributions first to increase the minimum value of their elements as in \citet{dvurechensky2018computational,lin2019complexity}) and the complexity analysis (see, e.g., Lemma \ref{lem:bound_lambda}).
\end{remark}

The optimal solution of the regularized multi-marginal optimal transport problem \eqref{eq:ot_multi_reg} can be compactly expressed in terms of the optimal variables of the dual problem. % $\Lambda = \{ \lambda_k \}_{k \in \Gamma}$.
More precisely, the optimal transport tensor is of the form
\begin{align} \label{eq:opt_sol}
  \hspace{-20pt} [\bB( \Lambda )] (x_1,\dots, x_m) = & \exp\left( -\bC(x_1,\dots, x_m)/ \eta\right) \nonumber \\
& \cdot  \prod_{k \in \Gamma}  \Big(  \exp \!\Big(\frac{\lambda_k (x_k)}{\eta}\Big) \mu_k(x_k)  \Big) ,
\end{align}
where $\Lambda = \{ \lambda_k \}_{k \in \Gamma}$ is the optimal solution of the dual of \eqref{eq:ot_multi_reg}, which is given by (cf. \cite{haasler2020tree})
\begin{equation} \label{eq:ot_multi_dual}
    \min_{ \Lambda } \psi( \Lambda) :=  \eta P( \bB(\Lambda) ) - \sum_{k\in \Gamma} \mu_k^\text{T} \lambda_k.
\end{equation}
Here, $P(\bB) = \sum_{\bx} \bB(\bx) \in \mathbb{R}$ is the projection over all marginals of $\bB$, i.e., the sum over all elements.

The optimal solution to \eqref{eq:ot_multi_dual} can be efficiently found by the renowned Sinkhorn iterations \citep{benamou2015iterative, haasler2021pgm}.
In particular, the multi-marginal Sinkhorn algorithm is to find the scaled variables $u_k= \exp(\lambda_k /\eta)$, for $k \in \Gamma$, %in \eqref{eq:opt_sol}
by iteratively updating them according to
\begin{equation} \label{eq:sinkhorn}
    u_k^{(t+1)} \leftarrow u_k^{(t)} \odot \mu_k ./ P_k( \bB(\Lambda^{(t)})). 
    % \text{ \ih{[Write element wise?]} {\red I agree, to avoid new notations}}
\end{equation}
There are several approaches to perform these updates: At each iteration, the next marginal $k\in \Gamma$ to be updated can be picked in a random, cyclic, or greedy fashion \citep{benamou2015iterative,lin2019complexity}. In this paper we discuss the random updating rule.
The greedy update requires more operations for each iteration as all the projections for $k\in\Gamma$ are needed for an update. The traditional cyclic update introduces strong couplings between updates which makes the complexity analysis much more challenging.

For general MOT, computing the projections $P_k(\bB(\Lambda^{(t)}))$ requires $\mathcal{O}(n^m)$ operations, which creates a large computational burden.
However, in case the MOT problem has a tree-structure, the projections $P_k( \bB(\Lambda^{(t)}) )$ can be computed by a message-passing algorithm that utilizes the belief propagation algorithm \citep{YedFreWei03}, as described in \cite{haasler2021pgm, haasler2020tree}. This requires only matrix-vector multiplications of size $n$.
In particular, the projections are of the form
%\ih{ [Wrote the algorithm from PGM paper without factor graphs:]}
\begin{align} 
    & [ P_k(  \bB(\Lambda^{(t)}) )](x_k)  = \\ & \begin{dcases} u_k^{(t)} (x_k) \mu_k(x_k) m_{\jiao{\ell_k} \to k} (x_k), & \quad \text{ if } k \in \Gamma \label{eq:proj_messages} \\
    \prod_{\ell \in N(k)} m_{\ell \to k} (x_k), & \quad \text{ if } k \notin \Gamma,  \end{dcases}
\end{align}
where the messages are computed as
% \begin{subequations}\label{eq:message_updates}
% \noeqref{eq:message_updates_a,eq:message_updates_b}
% \begin{align}
%     m_{\ell \to k} (x_k) & = \sum_{x_\ell} K^{(k,\ell)}(x_k,x_\ell) \prod_{j \in N(\ell) \setminus k} m_{j \to \ell} (x_\ell), \quad \text{ if } \ell \notin \Gamma \label{eq:message_updates_a} \\
%     m_{\ell \to k} (x_k) & = \sum_{x_\ell} K^{(k,\ell)}(x_k,x_\ell) u_\ell^{(t)}(x_\ell) \mu_\ell(x_\ell), \quad \text{ if } \ell \in \Gamma, \label{eq:message_updates_b}
% \end{align}
% \end{subequations}
\begin{align}
    & m_{\ell \to k} (x_k) 
    =\\
    & \begin{dcases}  \sum_{x_\ell} K^{(k,\ell)}(x_k,x_\ell) \prod_{j \in N(\ell) \setminus k} m_{j \to \ell} (x_\ell), \quad \text{ if } \ell \notin \Gamma \label{eq:message_updates} \\
    \sum_{x_\ell} K^{(k,\ell)}(x_k,x_\ell) u_\ell^{(t)}(x_\ell) \mu_\ell(x_\ell), \quad \text{ if } \ell \in \Gamma, 
    \end{dcases}
\end{align}
where $K^{(k,\ell)}(x_k,x_\ell)= \exp( - C^{(k,\ell)}(x_k,x_\ell)/ \eta)$.

Since we can without loss of generality assume that $\Gamma$ is the set of leaves of the tree, each vertex $k\in \Gamma$ has a unique neighbour $\ell_k \in N(k)$.
The Sinkhorn iterations \eqref{eq:sinkhorn} with the projections \eqref{eq:proj_messages} thus read
\begin{equation}
    u_k^{(t+1)}(x_k) \leftarrow ( m_{\ell_k \to k} (x_k) )^{-1}.
\end{equation}
%, where $\ell_k$ is the unique neighbour of $k$.
Note that when we update the scaling vectors $u_{k^{(t)}}^{(t)}$ and in the previous iteration updated $u_{k^{(t-1)}}^{(t-1)}$ it is only required to recompute the messages between $k^{(t-1)}$ and $k^{(t)}$ \citep{haasler2021pgm,SinHaaZha20}.
The Sinkhorn method is summarized in Algorithm~\ref{algo:sinkhorn}. Here, we apply a random updating scheme, where the next scaling vector to be updated is picked from a uniform distribution of the remaining scaling vectors, except the previous one. Other common update rules for the Sinkhorn iterations, such as cyclic or greedy updates, can be obtained by simply changing the selection of $k^{(t)}$ in Algorithm~\ref{algo:sinkhorn}.
\begin{algorithm}[t]
    \caption{SINKHORN\_BP($\epsilon', \{\mu_k \}_{k \in \Gamma}, \bC, \eta$)}
    \label{algo:sinkhorn}
    \begin{algorithmic}
        % \STATE{ {\bfseries Input:} Stopping tolerance $\epsilon'$ }
        \STATE{ {\bfseries Initialization:}  $u_k^{(0)} = \mathbf{1} \in \R^n $, for $k \in \Gamma$; \ $t=1$; \ $k^{(0)} \in \Gamma$}
        % (use factor marginal) same complexity
        \WHILE{ $\sum_{k \in \Gamma} \|P_k(\bB(\Lambda^{(t)})) - \mu_k \|_1 \geq \epsilon' $} 
        \STATE{1. Randomly pick $k^{(t)} \in \Gamma \setminus k^{(t-1)}$} 
        \STATE{2. Update messages $m_{k_1 \to k_2} $ according to \eqref{eq:message_updates} on the path from $k^{(t-1)}$ to $k^{(t)}$ }
        \STATE{3. Update $u_k^{(t+1)} (x_k)  $ to be $$\begin{cases} ( m_{\ell_k \to k} (x_k) )^{-1}, & \text{ for } k=k^{(t)}, \text{ and } \ell_k \in N(k), \\
                u_k^{(t)}(x_k),                & \text{ for } k \in \Gamma \setminus k^{(t)}, \end{cases} 
                % \quad \forall x_k
                $$}
        \STATE{4. $t \leftarrow t+1$ }
        \ENDWHILE
        \STATE { {\bfseries Output:}  $u_k^{(t+1)}$, $k \in \Gamma$ }
    \end{algorithmic}
\end{algorithm}

From the scaling vectors  $\{u_k\}_{k\in \Gamma}$ that are returned from Algorithm~\ref{algo:sinkhorn} we can construct the transport tensor $\widetilde \bB$ as in \eqref{eq:opt_sol}. However, this tensor is not guaranteed to lie in the feasible set $\PiGamma$, and thus a rounding step is needed. This is based on the rounding for bi-marginal optimal transport in \cite[Algorithm 2]{altschuler2017near}, and is stated in Algorithm~\ref{algo:round}.
Note that a transport tensor that solves a graph-structured MOT problem \eqref{eq:unregu mot} or \eqref{eq:ot_multi_reg} is fully determined by the projections $P_{k_1,k_2}(\bB)$ on the edges $(k_1,k_2)\in E$ \citep{KolFri09}, which are given by
\begin{equation} \label{eq:proj_bi}
    [P_{k_1,k_2}(\bB)] (x_{k_1}, x_{k_2}) = \sum_{\bx \setminus\{x_{k_1}, x_{k_2}\} } \bB( \bx).  
\end{equation}
%{\red Maybe we can make it more clear how to recover the full tensor from all these marginals.}
%\ih{ %[ For instance:]
%Namely, given the bi-marginals projections \eqref{eq:proj_bi} on all edges, the full tensor can be constructed as \jk{XX is this true? XX}
%\begin{equation}
 %   \bB(x_1,\dots, x_m) = \prod_{(k_1,k_2)\in E} [P_{k_1,k_2}(\bB)] (x_{k_1}, x_{k_2}).
%\end{equation}
%}
By slight abuse of notation, we let 
$\bB((B_{k_1,k_2})_{(k_1,k_2) \in E})$ denote this tensor that decouples according to the tree structure $G$ and
satisfies the projections $[P_{k_1,k_2}(\bB)] = B_{k_1,k_2}$ for $(k_1,k_2) \in E$  \citep{KolFri09}. Note that the projections \eqref{eq:proj_bi} can be cheaply computed from the scaling vectors $\{u_k\}_{k\in \Gamma}$ as described in \cite[Theorem~4]{haasler2021pgm}.
\begin{algorithm}[t]
    \caption{ROUND $(\bB,\{\mu_k\}_{k \in \Gamma} )$}
    \label{algo:round}
    \begin{algorithmic}
        \STATE {\bfseries Initialization:}  $\bB_{k,\ell_k}=P_{k,\ell_k}(\bB)\in \R^{n \times n} $ for all $k \in \Gamma$ and each $\ell_k \in N(k)$
        \FOR{$k\in \Gamma$}
        \STATE{Input $\left(\bB_{k,\ell_k};P_{\ell_k}(\bB),\mu_k \right)$ into \cite[Algorithm 2]{altschuler2017near} and get $\widehat{\bB}_{k, \ell_k}   $ such that $\widehat{\bB}_{k,\ell_k} \in \Pi(P_{\ell_k}(\bB),\mu_k)$}
        \ENDFOR
        % \ih{last term is all the stuff in the inside of the tree that is not changed by rounding algorithm. How to write this?}  
        \STATE{ {\bfseries Output:} %$   \widehat \bB_{k, \ell_k} (x_k, x_{\ell_k})$ for $k \in \Gamma$ 
        $ \widehat \bB = \{ \widehat \bB_{k, \ell_k}; k \in \Gamma \} \cup \{ P_{k_1,k_2} (\bB) ; (k_1,k_2)\in E, k_1,k_2 \notin \Gamma \}$   } 
        % \jk{XX output should depend on $\widehat \bB_{k, \ell_k}$. XX}
    \end{algorithmic}
\end{algorithm}
%
%\ih{Let ROUND return only matrices and explain in text how to construct $\widehat B$ from output and $[P_{k \notin \Gamma}(\bB)](\bx \setminus \{x_k,k\in \Gamma\})$.}
%
The full method for finding an $\epsilon$-approximate solution to a tree-structured MOT problem is summarized in Algorithm~\ref{algo:complete}.
\begin{algorithm}[t]
    \caption{$\epsilon$-approximation of tree-structured MOT }
    \label{algo:complete}
    \begin{algorithmic}
\STATE{ %{\bfseries Input:} 
$\eta \leftarrow \frac{\epsilon}{2 m \log(n)}$; \  $\epsilon' \leftarrow \frac{\epsilon }{8 {R_C^\Gamma}} $ .}
%\STATE {\bfseries Initialization:} $u_k^{(0)} = \mathbf{1} \in \R^n $, for $k \in \Gamma$. $t=1$, $k^{(0)} \in \Gamma$
\STATE{ %{\bfseries Step 1:} \st{Calculate $\widetilde{\bB}=$ }
$\{u_k\}_{k\in \Gamma} \leftarrow \text{SINKHORN\_BP} (\epsilon', \{\mu_k \}_{k \in \Gamma}, \bC, \eta)$. (Algorithm~\ref{algo:sinkhorn}) }
\STATE{Construct $\widetilde \bB(B_{k_1,k_2, (k_1,k_2) \in E})$ from $\{u_k\}_{k\in \Gamma}$}%$\widetilde \bB$ as in \eqref{eq:opt_sol} \ih{ [ Or: Construct $\widetilde\bB = \{ P_{k_1,k_2} (\widetilde \bB) ; (k_1,k_2) \in E \}$ from $\{u_k\}_{k\in \Gamma}$ ] } {\red I think either is fine}} 
\STATE{ %{\bfseries Step 2:} Round $\widehat{\bB}=$ 
$\widehat \bB \leftarrow \text{ROUND} (\widetilde{\bB}, \{\mu_k \}_{k \in \Gamma})$. (Algorithm~\ref{algo:round}) }
\STATE{ {\bfseries Output:} $\widehat{\bB}$ }

    \end{algorithmic}
\end{algorithm}

%{\red add the full algorithm here or the beginning of the next section?} \ih{I moved all Algorithms here. }

\section{Tree structured MOT analysis}
 In this section, we present a complexity bound for the Sinkhorn belief propagation algorithm for solving MOT problems with tree-structured costs. We first provide a few technical lemmas that will be used in the proof.  The proofs of all the supporting lemmas are given in the supplementary material.
The first result provides bounds for the scaling vector iterates.
% \subsection{Tree}
\begin{lemma} \label{lem:bound_lambda} 
    Let $\lambda_k = \eta \log(u_k)$, where $u_k$ are generated by Algorithm~\ref{algo:sinkhorn}. % \eqref{eq:sinkhorn}.
    Let $\Lambda^* = \{\lambda_k^*\}_{k\in \Gamma}$ be a solution of \eqref{eq:ot_multi_dual}.
    Then for each $k \in \Gamma$ it holds
    \begin{equation*}
        \begin{aligned}
            \max_{x_k} \lambda_k (x_k)  -\min_{x_k} \lambda_k (x_k)\leq R_C^k, \\
            \max_{x_k} \lambda_k^* (x_k)  -\min_{x_k} \lambda_k^* (x_k) \leq R_C^k,
        \end{aligned}
    \end{equation*}
    where
    \begin{equation*}
        R_C^k :=  \|C^{(k,\ell_k)}\|_\infty,
    \end{equation*}
    and where $\ell_k \in N(k)$ is the (unique) neighbour of $k$.
\end{lemma}
% \begin{proof}
%     %We use the bounds $ \exp(-\|C\|_\infty / \eta ) \leq K_{i_0 i_j} \leq 0$ for all $i_0,i_j \in \{1,\dots,n\}$, and $0 \leq (\mu_k)_i \leq 1$ for all $i=1,\dots,n$ and $k=1,\dots,m$.
%     Denote $ v_k(x_\ell) = \prod_{j\in N(\ell) \setminus k} m_{j \to \ell} (x_\ell)$, where $\ell \in N(k)$ is the unique neighbour of $k$, since $k$ is a leaf of the tree.
%     Assume variable $u_k$ was updated in the previous step of the algorithm. Then it holds
%     \begin{equation}
%         u_k(x_k) = 1 / m_{\ell \to k} (x_k)  = 1 / \left( K^{(k, \ell)} v_k \right) .
%     \end{equation}
%     Thus,
%     \begin{equation} \label{eq:lambda_bound_max}
%         \max_{x_k} \lambda_k (x_k) \leq - \eta \log\left( e^{-1 -\|C^{(k,\ell)}\|_\infty/\eta} v_k^T \ones  \right)
%         =  \eta + \|C^{(k,\ell)}\|_\infty - \eta \log\left( v_k^T \ones \right).
%     \end{equation}
%     Moreover,
%     \begin{equation}  \label{eq:lambda_bound_min}
%         \min_{x_k} \lambda_k (x_k) \geq - \eta  \log \left( e^{-1} ( v_k^T \ones )  \right) = \eta  - \eta \log\left( v_k^T \ones \right) .
%     \end{equation}

%     Combining \eqref{eq:lambda_bound_max} and \eqref{eq:lambda_bound_min} it follows
%     \begin{equation*}
%         \max_{x_k} \lambda_k (x_k)  -\min_{x_k} \lambda_k (x_k) \leq  \|C^{(k,\ell)}\|_\infty .
%     \end{equation*}

%     Note that the gradient of $\psi(\cdot)$ vanishes in $\Lambda^*$, since it is optimal to \eqref{eq:ot_multi_dual}. Thus, it holds $P_k(\bB(\Lambda^*) = \mu_k$ for $k=1,\dots,m$ and the bound for $\lambda_k^*$ follows in the same way as before.

% \end{proof}

The following Lemma relates the error in the dual objective value to the stopping criterion of Algorithm~\ref{algo:sinkhorn}.
\begin{lemma} \label{lem:psi2psi_star} 
    Let $\Lambda= \{\lambda_k\}_{k\in \Gamma}$, where $\lambda_k = \eta \log(u_k)$ and $u_k$ are generated by  Algorithm~\ref{algo:sinkhorn}, % {\red algorithm \eqref{eq:sinkhorn} good ref?}, 
    and let $\Lambda^* = \{ \lambda_k^*\}_{k\in\Gamma}$ be a solution to \eqref{eq:ot_multi_dual}. Then it holds
    \begin{equation}
        \psi(\Lambda) - \psi(\Lambda^*) \leq \RC   \sum_{k\in \Gamma} \| P_k(\bB(\Lambda)) - \mu_k \|_1 ,
    \end{equation}
    with $ \RC =  \max_{k\in \Gamma} R_C^k $, where $R_C^k$ is defined as in Lemma~\ref{lem:bound_lambda}.
\end{lemma}

The increment between two sequential Sinkhorn iterates is related to the stopping criterion of Algorithm~\ref{algo:sinkhorn} as described in the following.
\begin{lemma} \label{lemma:incremental psi} 
    For any $\Lambda^{(t)}$,
    let $\Lambda^{(t+1)}$ be the next iterate of the algorithm in \eqref{eq:sinkhorn}. 
    % \jiao{should be $(|\Gamma|-1)^2$ in denominator?}
    \begin{equation}\label{eq:iter_improve}
        \E \left[ \psi(\Lambda^{(t)}) - \psi(\Lambda^{(t+1)}) \right] \geq  \frac{\eta}{2 |\Gamma|^2} \left( e_t\right)^2,
    \end{equation}
    with
    \begin{equation}
        e_t:=\sum_{k \in \Gamma} \| P_k(\bB(\Lambda^{(t)})) - \mu_k \|_1.
    \end{equation}
    The expectation is over the uniform distribution of $k^{(t+1)} \in \Gamma \setminus k^{(t)}$.
\end{lemma}

We are now ready to state our first main result, which gives two probabilistic bounds on the required number of iterations in Algorithm~\ref{algo:sinkhorn}. 
\begin{theorem}
    \label{theo:t bound}
% Under the mild assumption $\eta \leq 0.5 |\Gamma| R_C^\Gamma$,
For sufficiently small $\eta$,
    Algorithm \ref{algo:sinkhorn} generates a tensor $\bB(\Lambda^{(t)})$ satisfying
    \begin{equation*}
        \sum_{ k \in \Gamma} \| P_k(\bB(\Lambda^{(t)})) - \mu_k \|_1 \leq \epsilon',
    \end{equation*}
    within $\tau$ iterations, where
    \begin{equation*}
        \E[\tau] \leq \frac{8 |\Gamma|^2 \RC }{\eta \epsilon'}.
    \end{equation*}
    Moreover, for any $\delta \in (0,0.5),$ 
    it holds that
    \begin{equation*}
        \P\left( \tau \leq \frac{48 |\Gamma|^2 \RC }{\eta \epsilon'} \log \frac{1}{\delta}\right) \geq 1-{\delta}.
        % R =  \max_k R_k \le 2 \|C\|_\infty \max_k \omega_k  - \eta \log \left(\min_{\substack{k=1,\dots,m,\\ i=1,\dots,n}} (\mu_k)_i \right)
    \end{equation*}
\end{theorem}
\begin{proof}[Proof sketch (see supplementary material for details)]
Define the stopping time $\tau:=\min \left\{t: e_t \leq \epsilon'\right\}$.
    Let $\{\mF_t := \sigma \left(\Lambda^{(1)},\ldots,\Lambda^{(t)} \right)\}_t$ be  the natural filtration.
    By Lemma \ref{lem:psi2psi_star} and Lemma \ref{lemma:incremental psi},
    \begin{align*}
       & \E \left[ \psi(\Lambda^{(t)}) - \psi(\Lambda^{(t+1)}) |\mF_t, t < \tau \right] \\
        & \geq  \frac{\eta}{2 |\Gamma|^2} \left(\max\left\{\frac{\psi(\Lambda^{(t)})-\psi(\Lambda^*)}{\RC }, \epsilon' \right\}\right)^2,
    \end{align*}
    Let
    %For shorthand, denote $\widetilde{\psi}(\Lambda^{(t)})=\psi(\Lambda^{(t)})-\psi(\Lambda^*)$, and let
    $\tau_1$ be the first iteration when $ \psi(\Lambda^{(t)})-\psi(\Lambda^*) \leq \RC \epsilon'$ and $\tau_2:=\tau-\tau_1 \geq 0$.
    % Define $Y_t=\tilde{\psi}(\Lambda^{(t)}) $ if $t \leq \tau_1$, and otherwise $\widetilde{\psi}(\Lambda^{(\tau_1)}) - \epsilon'^2 \eta (t -\tau_1)/ 2(|\Gamma|-1)^2$ if $t>\tau_1$. 
    We can bound $\tau_1$ and $\tau_2$ as
    \begin{equation*}
        \E[\tau_1]  \leq \frac{6 |\Gamma|^2 \RC}{\eta\epsilon'} - 1, \mbox{ and }\; 
        %+ \log_2 \ceil*{\frac{e_0}{\epsilon'}}
%    \end{equation*}
%    and
%    \begin{equation*}
        \E[\tau_2] 
        % \leq \E[\tau_2'] 
        \leq \frac{2|\Gamma|^2 \RC}{\eta\epsilon'} +1.
    \end{equation*}
%\jiao{The assumption implies $1+\log_2 \ceil*{\frac{e_0}{\epsilon'}} \leq \frac{2|\Gamma|^2 \RC}{\eta\epsilon'} $. }
    Summing up the two bounds results in the bound for $\E[\tau]$. The bound in probability follows similarly.\end{proof}

\begin{remark}
High probability bounds are often used in machine learning algorithms when randomness is involved.
Due to the logarithmic dependence $\log (1/\delta)$ in terms of the probability $1-\delta$, the high probability bound can safely be used as a surrogate of the deterministic bound. %Indeed, $\log (1/\delta)$ simply introduces an extra factor in the bound whose magnitude is not large for reasonably small $\delta$.
\end{remark}

In order to provide the complexity on the full method in Algorithm~\ref{algo:complete} we need the following two lemmas, which deal with the rounding method in Algorithm~\ref{algo:round}.
\begin{lemma} \label{lem:round bound}
    Let $\bB \in \R^{n^m}$, where $m\geq 3$, be a nonnegative $m$-mode tensor and $\{\mu_k\}_{k \in \Gamma}$ be a sequence of probability vectors, Algorithm \ref{algo:round} returns $\widehat{\bB}$ satisfying $P_k(\bB)=P_k(\widehat{\bB})$, for $k \in \Gamma$, and $P_k(\widehat{\bB})=\mu_k$, for $k \in \Gamma$.
    Moreover, it holds that
    \begin{equation*}
        \langle \bC, \bB\rangle-\langle\bC, \widehat{\bB}\rangle \leq  2\sum_{k\in \Gamma} \|C^{(k,\ell_k)}\|_{\infty} \|\mu_k-P_k({\bB}) \|_1,
    \end{equation*}
    where $\ell_k$ is the unique neighbour of $k$, for each $k \in \Gamma$.
\end{lemma}
% \begin{proof}
%     Due to the underlying tree structure of the problem it holds
%     \begin{align*}
%         \langle \bC, \bB\rangle-\langle\bC, \widehat{\bB}\rangle =  \sum_{(k_1,k_2)\in E } \langle C^{(k_1,k_2)}, P_{k_1,k_2}(\bB) - P_{k_1,k_2}(\widehat{\bB}) \rangle 
%     \end{align*}
%     By \holder  inequality and \cite[Lemma 7]{altschuler2017near},
%     \begin{align}
%         \langle \bC, \bB\rangle-\langle\bC, \widehat{\bB}\rangle
%         & \leq  \sum_{(k_1,k_2)\in E}  \| C^{(k_1,k_2)}\|_{\infty}  \| P_{k_1,k_2}(\bB) - P_{k_1,k_2}(\widehat{\bB}) \|_1    \\
%         & \leq  2 \sum_{k \in \Gamma} \|C^{(k,\ell_k)}\|_{\infty} \|\mu_k-P_k({\bB}) \|_1,
%     \end{align}
%     In the second step it is used that in the case $k_1,k_2 \notin \Gamma$ the bound in \cite[Lemma 7]{altschuler2017near} becomes
%     \begin{equation}
%      \| P_{k_1,k_2}(\bB) - P_{k_1,k_2}(\widehat{\bB}) \|_1 \leq 2 \left(  \| P_{k_1}(\bB) - P_{k_1}(\bB) \| +  \| P_{k_2}(\bB) - P_{k_2}(\bB) \| \right) = 0.
%     \end{equation}
% \end{proof}

\begin{lemma} 
    \label{lem:Bhat_Bstar}
    Let $\widetilde{\bB}$ be the output of Algorithm \ref{algo:sinkhorn}, let $ \widehat{\bB} $ be the output of Algorithm~\ref{algo:round} with input $(\widetilde{\bB}, \{\mu_k\})$, and let $\bB^*$ denote the optimal solution to the unregularized MOT problem \eqref{eq:unregu mot}. Then it holds that
    \begin{align*}
        \<\bC, \widehat{\bB}\>-\<\bC, \bB^*\> \leq & \ m \eta \log(n) \\
        & + 4 \sum_{k\in \Gamma} \|C^{(k,\ell_k)}\|_{\infty} \|\mu_k - P_k(\widetilde{\bB})\|_1.
    \end{align*}
\end{lemma}

We now have the tools to state our new complexity bound for finding $\epsilon$-approximate solutions to tree-structured MOT problems.  Denote by $d(G)$ the maximum distance of two nodes in the graph $G$.
\begin{theorem}\label{thm:opercomplexity}
    Algorithm~\ref{algo:complete} finds an $\epsilon$-approximate solution to the tree-structured MOT problem \eqref{eq:unregu mot} in $T$ arithmetic operations, where
    \begin{equation*}
        \E[T] = \mathcal{O} \left( \frac{d(G) m |\Gamma|^2 n^2 (\RC)^2 \log(n)}{\epsilon^2} \right).
    \end{equation*}
    Moreover, for all $\delta \in (0,0.5)$ it holds that
    % \jiao{miss a very large coefficients $48c \times 2\times 8$??}
    \begin{equation*}
        \P\left(  T \leq \frac{ c d(G) m |\Gamma|^2 n^2 (\RC)^2 \log(n) \log(1/\delta) }{\epsilon^2}  \right) \geq 1-{\delta}
    \end{equation*}
    where $c$ is a universal constant.
%
    %\st{$\mathcal{O} \left( \frac{d(G) m |\Gamma|^2 n^2 (\RC)^2 \log(n)}{\epsilon^2} \right)$
    %arithmetic operations, where $d(G)$ is the maximum distance of two nodes in the graph $G$.}  
\end{theorem}
\begin{proof}
    With the specific choices $\eta=\frac{\epsilon}{2 m \log(n)}$ and $\epsilon'= \frac{\epsilon }{8 \RC}$
    % \begin{equation}
    %     \eta=\frac{\epsilon}{2 m \log(n)}, \qquad \epsilon'= \frac{\epsilon }{8 \RC},
    % \end{equation}
    we get $\langle \bC, {\widehat{\bB}} \rangle -\langle \bC, {\bB^*} \rangle \leq \epsilon$.
    % \begin{align*}
    %     \langle \bC, {\widehat{\bB}} \rangle -\langle \bC, {\bB^*} \rangle \leq \epsilon.
    % \end{align*}
    By Theorem~\ref{theo:t bound}, the stopping time $\tau$ satisfies
    \begin{align*}
         \E [\tau] = \frac{8 |\Gamma|^2 \RC}{\eta \epsilon'} = \mathcal{O} \left(\frac{ m |\Gamma|^2 (\RC)^2 \log(n)}{\epsilon^2} \right).
    \end{align*}

    Since in each iteration of Algorithm~\ref{algo:sinkhorn} the messages between two leave nodes of the tree are updated, and each message update is of complexity $\mathcal{O}(n^2)$, one iteration takes at most $\mathcal{O}( d(G) n^{2}) $ operations. Thus, in expectation, a solution is achieved in
    \begin{equation}
        \mathcal{O} \left(\frac{ d(G)  n^2 m |\Gamma|^2 (\RC)^2 \log(n)}{\epsilon^2} \right)
    \end{equation}
    operations.
    Algorithm \ref{algo:round} takes $\mathcal{O}(|\Gamma|n^2)$ (see Lemma 7 in \citet{altschuler2017near}). Hence, the bound on $\E[T]$ follows. The bound in probability follows similarly.
    %the total complexity bound is $\mathcal{O}( d(G) m |\Gamma|^2 n^2 (\RC)^2 \log(n)\epsilon^{-2} )$.
\end{proof}

% \ih{ [Added some stuff from supplementary material to focus more on the general graph case]}
% \jiao{[Can we add a subsection for this? For example one subsection for above tree texts, one subsection for general graph.]}
% {\red we need to have at least two subsections if we want to use subsections}
\section{Extension to general graphs}
For a general graph, we cannot directly apply the belief propagation algorithm. One way to tackle this is to construct a tree factorization over the graph. %, and apply the junction tree algorithm.
A junction tree (also called tree decomposition) describes a partitioning of a graph, where several nodes are clustered together, such that the interactions between the clusters can be described by a tree.
A cluster $c$ is a collection of nodes, and we write $\bx_c = \{ x_k, k\in c\}$. Moreover, the matrices $K^{(k_1,k_2)}= \exp(- C^{(k_1,k_2)}/\eta)$, for $(k_1,k_2)\in E$, can be understood as pair-wise potentials. A junction tree is then defined as follows.
\begin{defn}
A junction tree $\mathcal{T}=(\mathcal{C},\mathcal{E})$ over a graph $G=(V,E)$ is a tree whose nodes $c\in \mathcal{C}$ are associated with subsets $\bx_c \subset V$, and that satisfies the following properties:\\[-20pt]
\begin{itemize}
\setlength{\itemsep}{1pt}
\setlength{\parskip}{1pt}
    \item Family preservation: For each potential $K$ there is a cluster $c$ such that $\text{domain}(K) \subset \bx_c$.
    \item Running intersection: For every pair of clusters $c_i, c_j \in \mathcal{C}$, every cluster on the path between $c_i$ and $c_j$ contains $\bx_{c_i} \cap \bx_{c_j}$.
\end{itemize}
For two adjoining clusters $c_i$ and $c_j$, we define the separation set $S_{ij}= \{ v\in V : v\in c_i \cap c_j \} $.
\end{defn}
It is often practical to find a junction tree that is as similar to a tree as possible. A measure of this is given by the following definition.
\begin{defn}
 For a junction tree $\mathcal{T}=(\mathcal{C},\mathcal{E})$, we define its width as $$\text{width}(\mathcal{T}) = \max_{c\in \mathcal{C}} |c| -1.$$
 For a graph $G$, we define its tree-width as $$w(G) = \min\{ \text{width}(\mathcal{T}) \ |\ \mathcal{T} \text{ is a junction tree for } G\}.$$
\end{defn}
Based on this partitioning, we achieve the following complexity bound for general graph-structured MOT problems.
The derivation of the modified algorithm is deferred to \jiao{Section \ref{sec:general_graph}.}
% the supplementary material.
%
\begin{theorem}\label{thm:general_complexity}
  Let $R = \max_{k\in \Gamma} R^k_\bC$, where $R^k_\bC = \|\bC_{c_{\ell_k}}( \bx_{c_{\ell_k}})\|_\infty$, and $c_{\ell_k}$ is the neighbouring clique to $c_k$. 
  A generalization of Algorithm~\ref{algo:complete} %with message passing \eqref{eq:message_clusters} 
  finds an $\epsilon$-approximate solution to the general graph structured MOT problem \eqref{eq:unregu mot} in $T$ arithmetic operations, where
    \begin{equation}
    \label{eq:ET_general}
        \E[T] = \mathcal{O} \left( \frac{d(\mathcal{T}) m |\Gamma|^2 n^{w(G)+1} R^2 \log(n)}{\epsilon^2} \right).
    \end{equation}
    Moreover, for all $\delta \in (0,0.5)$ there exists a universal constant $c$ such that
    \begin{equation*}
    \begin{aligned}
        \P\left(  T \leq \frac{c d(\mathcal{T}) m |\Gamma|^2 n^{w(G)+1} R^2 \log(n) \log(1/\delta) }{\epsilon^2}  \right) \\ \geq 1-{\delta}.
    \end{aligned}
    \end{equation*}
\end{theorem}

In Algorithm \ref{algo:sinkhorn}, the per iteration complexity is not independent of the random choice of the update, and thus not independent of the number of iterations.
The results in Theorem~\ref{thm:opercomplexity} and \ref{thm:general_complexity} thus depend on the maximum iteration complexity, and can be improved by utilizing the expected (average) iteration complexity.
Therefore, let $\bar d(G)$ denote the average distance between any two nodes in $\Gamma$.
\begin{theorem}\label{thm:general_average_dist}
    A generalization of Algorithm~\ref{algo:complete} finds an $\epsilon$-approximate solution to the graph-structured MOT problem \eqref{eq:unregu mot} in $T$ arithmetic operations, where 
    % \ih{Change constant to R}
    \begin{equation*}
        \E[T] = \mathcal{O} \left( \frac{\bar d(\mathcal{T}) m |\Gamma|^2 n^{w(G)+1} R^2 \log(n)}{\epsilon^2} \right).
    \end{equation*}
    % Moreover, for all \jiao{$\delta \in (0,0.25)$} there exists a universal constant $c$ such that
    % \begin{equation*}
    %     \P\left(  T \leq \frac{c \bar d(G) m |\Gamma|^2 n^2 (\RC)^2 \log(n) \log(1/\delta) }{\epsilon^2}  \right) \geq 1-{\delta}.
    % \end{equation*}
    %Here $\bar d(G)$ denotes the average distance between any two nodes in $\Gamma$.
\end{theorem}

% \begin{theorem}\label{thm:tree_average_dist}
%     Algorithm~\ref{algo:complete} finds an $\epsilon$-approximate solution to the tree-structured MOT problem \eqref{eq:unregu mot} in $T$ arithmetic operations, where
%     \begin{equation*}
%         \E[T] = \mathcal{O} \left( \frac{\bar d(G) m |\Gamma|^2 n^2 (\RC)^2 \log(n)}{\epsilon^2} \right).
%     \end{equation*}
%     % Moreover, for all \jiao{$\delta \in (0,0.25)$} there exists a universal constant $c$ such that
%     % \begin{equation*}
%     %     \P\left(  T \leq \frac{c \bar d(G) m |\Gamma|^2 n^2 (\RC)^2 \log(n) \log(1/\delta) }{\epsilon^2}  \right) \geq 1-{\delta}.
%     % \end{equation*}
%     %Here $\bar d(G)$ denotes the average distance between any two nodes in $\Gamma$.
% \end{theorem}

\section{Discussion of results} \label{sec:discussion}

We consider a class of tree-structured MOT problems, which contains many MOT applications of interest.
%we call balanced MOT problems.
\begin{defn}
\label{def:balanced_tree}
 Given a sequence of tree-structured MOT problems, where the number of nodes go to infinity,
 we call the sequence of such problems balanced if there is a constant $c$ such that $|\Gamma| \RC \le c\|\bC\|_\infty$.
\end{defn}
Many MOT problems that arise in practice are balanced, see 
% the supplementary material
\jiao{Section \ref{sec:balanced_tree}}
for a number of examples.
From Theorem~\ref{thm:opercomplexity} it follows that Algorithm~\ref{algo:complete} finds an $\epsilon$-approximate solution to balanced MOT problem \eqref{eq:unregu mot} in $T$ operations, where
\begin{equation}
    \E[T] = \mathcal{ O}  \left( \frac{\bar d(G) m  n^2 \|\bC\|_\infty^2 \log(n)}{\epsilon^2} \right).
\end{equation}

% \begin{remark} \label{cor:bound_mot}
% In many cases $\RC = \mathcal{O} (|\Gamma|^{-1} \|\bC\|_\infty)$ (see discussion in the Appendix \ref{sec:balanced_tree}).
% %If the terms $\|C^{(k,\ell_k)} \|_\infty$, where $\ell_k$ is the neighbour of $k\in \Gamma$, are of the same order for all $k \in \Gamma$, t
% Then from Theorem~\ref{thm:opercomplexity}, 
% Algorithm~\ref{algo:complete} finds an $\epsilon$-approximate solution to the tree-structured MOT problem \eqref{eq:unregu mot} in $T$ operations, where
% \begin{equation}
%     \E[T] = \mathcal{ O}  \left( \frac{d(G) m  n^2 \|\bC\|_\infty^2 \log(n)}{\epsilon^2} \right).
% \end{equation}
% \end{remark}
% Remark~\ref{cor:bound_mot}
This lets us compare our result with the bound for general MOT problems in \citet{lin2019complexity} without acceleration, which is given by $ \mathcal{ O}  \left(  m^3  n^{m} \|\bC\|_\infty^2 \log(n)\epsilon^{-2} \right)$.
Moreover, when the MOT problem on the junction tree is balanced, 
% $\RC = \mathcal{O} (|\Gamma|^{-1} \|\bC\|_\infty)$. Then 
by a similar argumentation the expectation bound in Theorem~\ref{thm:general_complexity} can be given by
%\eqref{eq:ET_general} could be replaced by 
% \jiao{don't use that equation just say it's balanced in words}
\begin{equation}
 \E[T] = \mathcal{ O}  \left( \frac{\bar d(\mathcal{T}) m  n^{w(G)+1} \|\bC\|_\infty^2 \log(n)}{\epsilon^2} \right).
\end{equation}

%\begin{remark}
%\label{rem:general_graph}
%The methods presented in this paper can be generalized to any graph $G$ that is not necessarily a tree.
%In this case each iteration of ISBP takes $\mathcal{O}( d(G) n^{w(G)+1} )$, where $w(G)$ is the tree-width of $G$. The detailed discussions are in Appendix \ref{sec:general_graph}. Thus, a generalization of Algorithm~\ref{algo:complete} is expected to return an $\epsilon$-approximate solution to the MOT problem in $\mathcal{ O}  \left( \frac{d(G) m  n^{w(G)+1} \|\bC\|^2_\infty \log(n)}{\epsilon^2} \right)$ with Remark~\ref{cor:bound_mot}. For a complete graph this matches the best known result without acceleration in \citet{lin2019complexity}.
%\end{remark}

In case the underlying graph is fully connected and not balanced we have $\bar d(\mathcal{T})=2$, $w(G)=m-1$ and $R=\|\bC\|_\infty$ in Theorem~\ref{thm:general_average_dist}.
In fact, in this case our algorithm does not exploit any graph structures, and thus the complexity is the same for general cost tensors that do not decouple into pairwise terms as in \eqref{eq:C_graph}.
Thus, the complexity of Algorithm~\ref{algo:complete} for general MOT problems matches the bound for general MOT problems in \cite{lin2019complexity}.

%\begin{remark} \label{rem:barycenter}
Consider the barycenter problem introduced in Example~\ref{exa: barycenter}. This problem is a MOT problem \eqref{eq:unregu mot} with underlying graph as illustrated in Figure~\ref{fig:barycenter_graph}. Here, $d(G)=2$, $|\Gamma|=L$, and $m=L+1$. Moreover, by \eqref{eq:barycentercost}, we have $ \RC = \frac{1}{L} \|C\|_\infty $.
Thus, Algorithm~\ref{algo:complete} is expected to return an $\epsilon-$approximate solution to problem \ref{eq:barycenter_pairwise} in $\mathcal{O}( L n^2 \|C\|_\infty^2 \log(n)\epsilon^{-2} )$.
This coincides with the best known bound for the barycenter problem \citep{kroshnin2019complexity,lin2020fixedsupport} without acceleration.
In fact, the argument can be extended to the case of non-uniform weights in the barycenter problem \eqref{eq:barycenter_pairwise}, see \jiao{Section \ref{sec:barycenter_nonuniform}.}
% supplementary material.
We also point out that the regularizer used in the Wasserstein barycenter literature is different from ours: one is pairwise regularization and one is regularization over the full tensor $\bB$. For more details on this comparison, see \citet{haasler2020tree}[Section 5].
% %
% %\jiao{Consider the Example \ref{exa: barycenter}, the graph associated with it is a star-shape graph. Obviously, the maximum distance of two nodes is $d(G)=2$ and the number of fixed marginals is $|\Gamma|=K$. So the total complexity bound for the $\epsilon-$approximation of problem \ref{eq:barycenter_pairwise} is $\mathcal{O}( K n^2 R \log(n)\epsilon^{-2} )$}
%\end{remark}

\section{Experiments}\label{sec:exp}

% \ih{[The text inside the figures is really small. Can we increase the font size?]}
% \jiao{[I didn't add rounding...
% I didn't do random update because I want to control randomness when comparing with brutal force Sinkhorn.
% ]} 
% \ih{[What is random then?]}

We show numerical results for three types of MOT problems.
We consider the barycenter problem in Example~\ref{exa: barycenter}, which is structured according to the graph in Figure~\ref{fig:barycenter_graph}, and the Hidden Markov Model example in 
%We conducted the barycenter (Example \ref{exa: barycenter}) and HMM example in
\citet[Section V.B]{haasler2021pgm}, which is structured according to the graph in Figure~\ref{fig:HMM_graph}.
In particular, these two types of problems are tree-structured.
\begin{figure}[tb]
 \centering
 \begin{tikzpicture}
  \tikzstyle{main}=[circle, minimum size = 15pt, thick, draw =black!80, node distance = 15pt]
  \node[main] (mu1) {};
  \node[main] (mu2) [right=of mu1] {};
  \node[main] (mu3) [right=of mu2] {};
  \node[] (mu4) [right=of mu3] {};
  \node[] (muTm1) [right=of mu4] {};  
  \node[main] (muT) [right=of muTm1] {};
  
  \node[main,fill=black!10] (phi1) [below=of mu1] {};
  \node[main,fill=black!10] (phi2) [right=of phi1,below=of mu2] {};
  \node[main,fill=black!10] (phi3) [right=of phi2,below=of mu3] {};
  \node[] (phi4) [right=of phi3] {};
  \node[] (phiTm1) [right=of phi4] {};
  \node[main,fill=black!10] (phiT) [right=of phi3,below=of muT] {};

  \draw[-, thick] (mu1) --   (mu2);
  \draw[-, thick] (mu2) --   (mu3);
  \draw[-, thick] (mu3) --   (mu4); 
  \draw[loosely dotted, very thick] (mu4) -- (muTm1); 
  \draw[-, thick] (muTm1) --  (muT); 
  
  \draw[-, thick] (mu1) --  (phi1);
  \draw[-, thick] (mu2) --  (phi2);
  \draw[-, thick] (mu3) --  (phi3);
  \draw[-, thick] (muT) --  (phiT);
  \draw[loosely dotted, very thick] (phi4) -- (phiTm1); 
 \end{tikzpicture}
 \caption{Graph associated with a Hidden Markov Model.}
\label{fig:HMM_graph}
\end{figure}
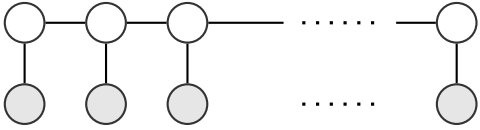
% \begin{figure}[tb]
% \centering
%     \begin{tikzpicture}
%     \tikzstyle{circ}=[circle, minimum size = 8mm, thick, draw =black!80, node distance = 15mm and 15mm]
%     \node[circ, fill=black!10] (c1)  {$\mu_1$};
% 	\node[circ, fill=black!10] (c2) [right=of c1] {$x_2$};
% 	\node[] (c3) [right=of c2] {};
% 	\node[circ, fill=black!10] (cN) [right=of c3] {$x_N$};
% 	\node[circ, node distance = 15mm and 5mm] (c0) [above right=of c1] {$x_0$};
% 	\node[circ, node distance = 15mm and 5mm] (cNp1) [above left=of cN] {$\mu_{\!N\!+\!1\!}$};
% 	\draw[loosely dotted] (c2) -- (cN);
% 	\draw (c0) -- (c1) -- (cNp1) -- (c0);
% 	\draw (c0) -- (c2) -- (cNp1);
% 	\draw (c0) -- (cN) -- (cNp1);
% 	\end{tikzpicture}
% \caption{Graph and factor graph for the cost in the Wasserstein least square problem.}
% 	\label{fig:W2leastsquare_factor_graph}
% \end{figure}
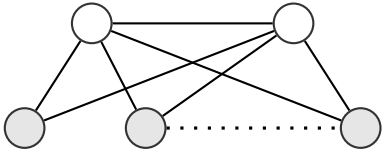
\begin{figure}[tb]
\centering
    \begin{tikzpicture}
    \tikzstyle{circ}=[circle, minimum size = 15pt, thick, draw =black!80, node distance = 30pt and 30pt]
    \node[circ, fill=black!10] (c1)  {};
	\node[circ, fill=black!10] (c2) [right=of c1] {};
	\node[] (c3) [right=of c2] {};
	\node[circ, fill=black!10] (cN) [right=of c3] {};
	\node[circ, node distance = 10mm and 5mm] (c0) [above right=of c1] {};
	\node[circ, node distance = 10mm and 5mm] (cNp1) [above left=of cN] {};
	\draw[loosely dotted, very thick] (c2) -- (cN);
	\draw[-, thick] (c0) -- (c1) -- (cNp1) -- (c0);
	\draw[-, thick] (c0) -- (c2) -- (cNp1);
	\draw[-, thick] (c0) -- (cN) -- (cNp1);
	\end{tikzpicture}
\caption{Graph associated with a Wasserstein least square problem.}
	\label{fig:W2leastsquare_factor_graph}
\end{figure}
The third example is the Wasserstein least square problem \citep{karimi2020statistical}, which is associated with the graph in Figure \ref{fig:W2leastsquare_factor_graph}.
Note that this is a graph with tree-width two.
In all the above graphs, gray nodes correspond to fixed marginals $\{\mu_i \}_{k \in \Gamma}$, and white nodes are estimated in the problem.
% \ih{[What is the setup for the general problem? ]}

The cost %$C(x_{L+1}, x_\ell)$
matrices $C^{(k_1,k_2)}$ in \eqref{eq:C_graph} are set to be the squared Euclidean distance.
The constrained marginal distributions $\{\mu_k \}_{k \in \Gamma}$ are supported on a uniform grid with $n$ points between 0 and 1, where the values are generated from the log-normal distribution and normalized to sum to one.
% \st{are randomly generated normalized standard Gaussian vectors and supported on uniform grids between 0 and 1.}
% \ih{[What are the marginals?]}
We choose the accuracy $\epsilon=0.5$ in Algorithm~\ref{algo:complete}. %, $\eta=\frac{\epsilon}{2m\log(n)}$, $\epsilon'=\frac{\epsilon}{8R_C^\Gamma}$ to be consistent with the Algorithm 3 in our paper.
As a comparison, we implemented a brutal force Sinkhorn method, which
%is only different from our Sinkhorn-BP algorithm in calculating $P_k({\bf{B}}(\Lambda^{(t)}))$. The brutal force Sinkhorn calculates $P_k({\bf{B}}(\Lambda^{(t)}))$ by summing up all the values in ${\bf{B}}$ except the projection dimension.
computes the projections $P_k({\bf{B}}(\Lambda^{(t)}))$ in the Sinkhorn iterates \eqref{eq:sinkhorn} by directly summing over the elements of the tensor ${\bf{B}}(\Lambda^{(t)})$ as in \eqref{eq:proj_bruteforce}.
%On the other hand, the Sinkhorn belief propagation algorithm computes the projections by performing the message passing routine in \eqref{eq:proj_graph}.
We use a random update rule for both methods.
The number of iterations of both brutal force Sinkhorn and Sinkhorn belief propagation are nearly
% \ih{[Still only nearly? \jiao{it's like in 25 times run, brutal force and bp have same iterations for 24 times. I checked the optimality gap, they are the same in 4 digits after dot. Can we cross out "nearly"?}]} 
the same. 
% For both algorithms, we use the code given by \url{https://github.com/qshzh/cbp}.
We repeat every experiment 5 times with different random seeds and report the total run time in Figure \ref{fig:exp}. The theoretical complexity bound is also presented as dashed lines. The run time of brutal force Sinkhorn grows in a higher polynomial of $n$ and grows exponentially with respect to $m$. This coincides with the general MOT bound and our bounds in Table \ref{tab:bounds}. We can also tell our bound is a bit pessimistic about the dependence over $n$.
% and report the mean and standard deviation. \jiao{add some analysis!} 
% \ih{[Is there anything random in the setup? If not, does it make sense to report several runs?]}
% To completely show our algorithm advantage of accelerating projection, we not only show the total run time but also show the run time per iteration in the Barycenter example.
\begin{figure}[t]
   \centering
%   \begin{subfigure}{0.5\textwidth}
%       \centering
%       \includegraphics[width=0.48\linewidth]{images/barycenter_fix_m.png}   \includegraphics[width=0.48\linewidth]{images/barycenter_fix_n.png}
%       \caption{Barycenter.}
%      \label{fig:barycenter}
%   \end{subfigure}
%   \begin{subfigure}{0.5\textwidth}
% \centering
% \includegraphics[width=0.48\linewidth]{images/hmm_fix_m.png}
% \includegraphics[width=0.48\linewidth]{images/hmm_fix_n.png} 
% \caption{Hidden Markov Model.}
%  \label{fig:hmm}
%   \end{subfigure}
% \begin{subfigure}{0.5\textwidth}
% \centering
% \includegraphics[width=0.48\linewidth]{images/w2_ls_fix_m.png}
% \includegraphics[width=0.48\linewidth]{images/w2_ls_fix_n.png} 
% \caption{Wasserstein least square.}
%  \label{fig:general}
% \end{subfigure}
   \begin{subfigure}{0.5\textwidth}
       \centering
      \includegraphics[width=0.48\linewidth]{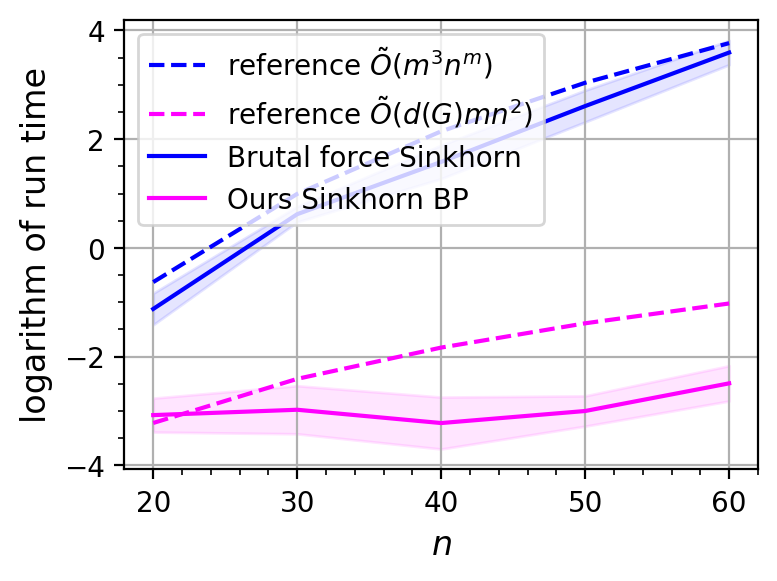}   \includegraphics[width=0.48\linewidth]{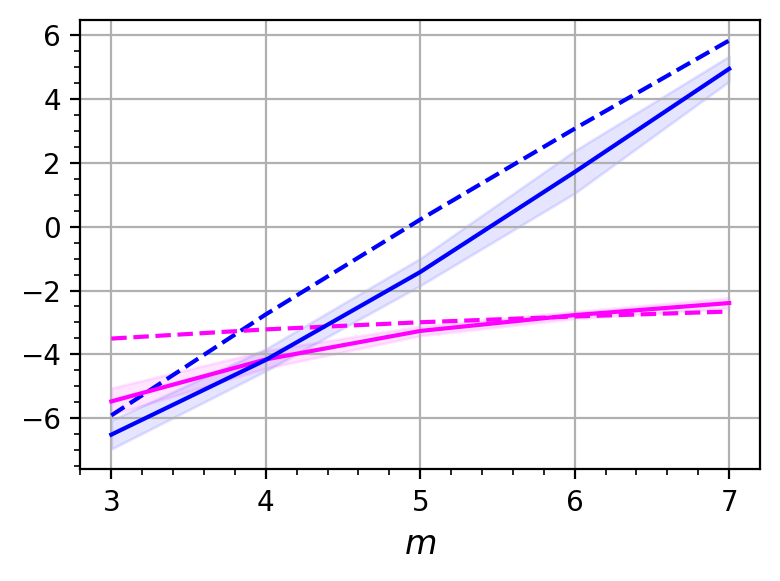}
      \caption{Barycenter.}
     \label{fig:barycenter}
   \end{subfigure}
   \begin{subfigure}{0.5\textwidth}
\centering
\includegraphics[width=0.48\linewidth]{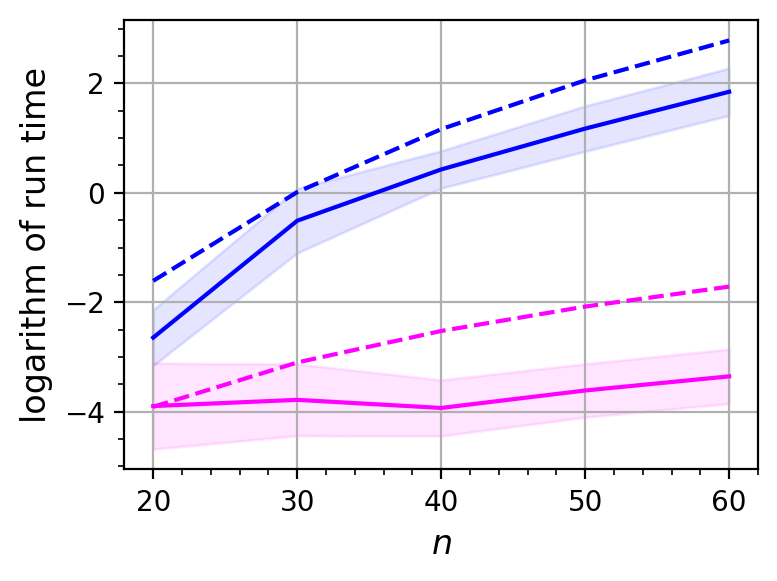}
\includegraphics[width=0.48\linewidth]{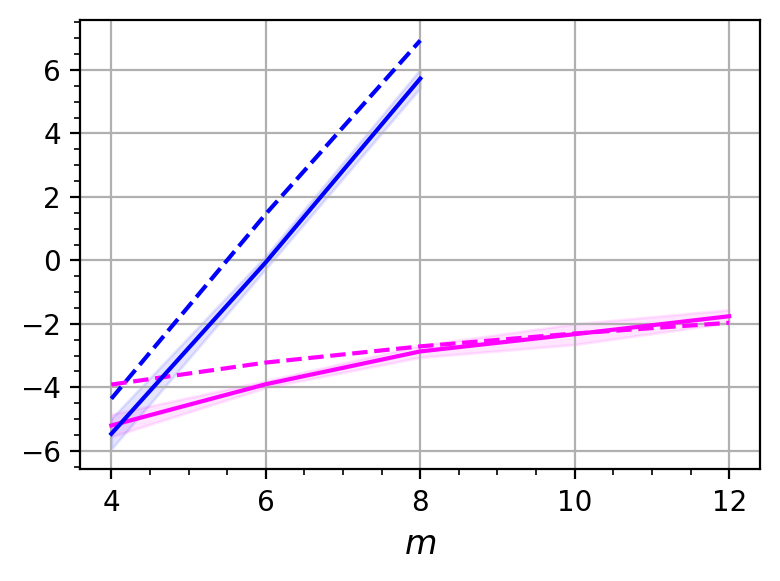} 
\caption{Hidden Markov Model.}
 \label{fig:hmm}
   \end{subfigure}
\begin{subfigure}{0.5\textwidth}
\centering
\includegraphics[width=0.48\linewidth]{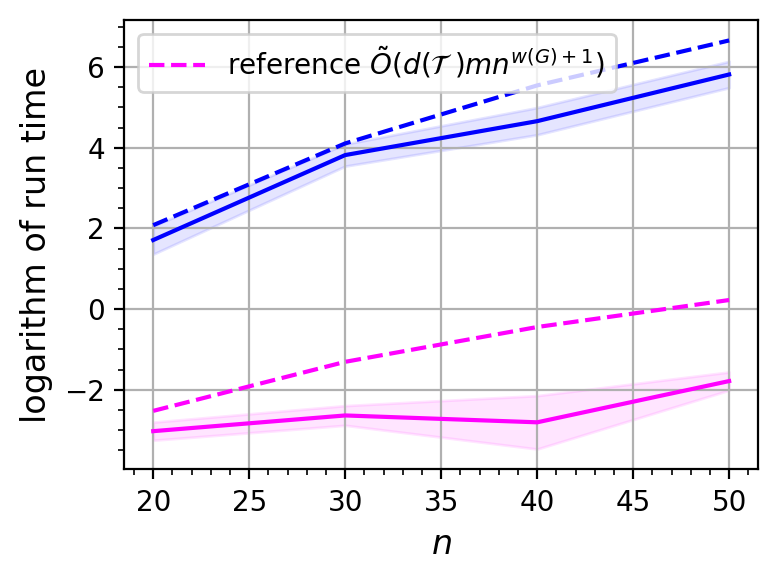}
\includegraphics[width=0.48\linewidth]{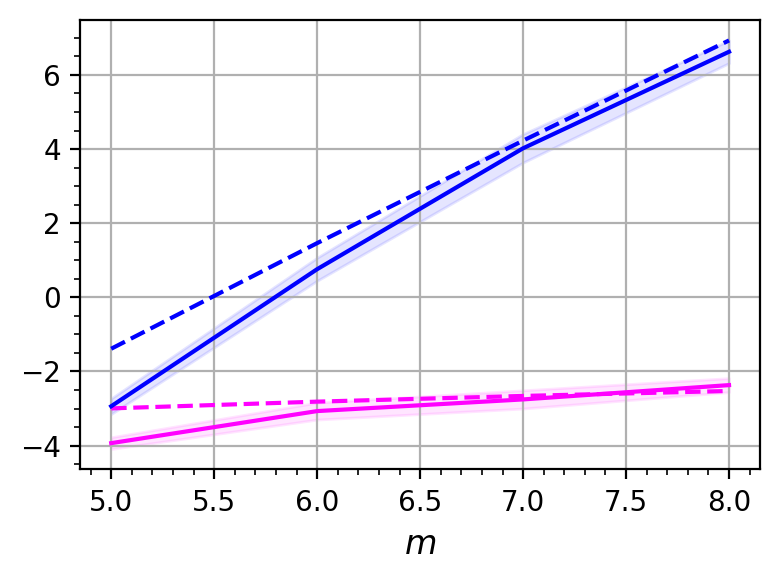} 
\caption{Wasserstein least square.}
 \label{fig:general}
\end{subfigure}
\caption{Logarithm of total run time in seconds for brutal force Sinkhorn and Sinkhorn belief propagation.
The left column shows the run time as a function of $n$ when $m$ is fixed ($m=4$ in (a) and (b); $m=5$ in (c)), and the right column vice versa with $n=10$.}
%We fix $m$ and study the run time dependence on $n$ in the left column and flip in the right column.}
 \label{fig:exp}
\end{figure}

\section{Conclusion}
In this work we considered a class of multi-marginal optimal transport problems where the cost functions can be decomposed according to a graph. It turns out that the computational complexity of MOT can be significantly reduced by exploiting the graphical structures. More specifically, without any structure, the complexity grows exponentially as the number of marginals increases. With graphical structure, the dependence becomes polynomial.
{We provide a complexity bound $ \mathcal{\tilde O}(d(\mathcal{T})m n^{w(G)+1}\epsilon^{-2})$ for solving graph-structured MOT problems}
%\ih{
% [Give bound for general graphs] 
%We provide a complexity bound $ \mathcal{\tilde O}(\bar d(G) m |\Gamma|^2n^2\epsilon^{-2})$ for solving the tree-structured MOT problems
based on the Sinkhorn belief propagation algorithm \citep{haasler2021pgm,SinHaaZha20} with the random updating rule.
%We also provide a similar bound for general graphical-structured MOT problems}.
One limitation of the present work is that the proof techniques do not seem to be applicable to Sinkhorn iterations with cyclic updating rule, which is the most popular strategy used in practice. This will be a future research direction. We also plan to accelerate the Sinkhorn belief propagation algorithm using ideas from \citet{lin2019complexity,kroshnin2019complexity}. 
% \ih{ [Don't need this for AISTATS??] \st{This work is a pure theoretical study and we don't foresee any negative societal impacts.}}

\bibliography{complexity_bound}
\bibliographystyle{icml2021}

% \end{document}

\onecolumn
\newpage
\appendix

% \ih{ [There is some issue with references to the main paper. To work around this, when the labels in the main document are changed, the aux file has to be constructed and uploaded manually.] }

\section{The dual of the regularized MOT problem and the Sinkhorn iterations}

In this section we provide details to Section~3. In particular, we derive the dual of the regularized MOT problem and the Sinkhorn belief propagation algorithm.

% In particular, we introduce the barrier term
% \begin{equation}
%     H(\bB \mid \bM)=\langle\bB, \log(\bB)-\log(\bM) -\ones_{n^m}\rangle,
% \end{equation}
% where
% \begin{equation}
%     \bM(x_1,x_2,\ldots, x_m)=\prod_{k\in \Gamma} \mu_k(x_k).
% \end{equation}
% The regularized MOT problem reads then
% \begin{align} 
%     \min_{ \bB \in \PiGamma } \langle \bC, \bB \rangle + \eta H(\bB\mid \bM)
% \end{align}
% \jiao{[Do we need more details for dual derivation below??]}
The Lagrangian function of problem \eqref{eq:ot_multi_reg} is
\begin{align}
\label{eq:lag_function}
    L(\bB, \Lambda)=\langle \bC, \bB \rangle + \eta H(\bB\mid \bM) - \sum_{k\in \Gamma} \lambda_k^\tT \left(P_k(\bB) -\mu_k \right),
\end{align}
where $\Lambda=(\lambda_k)_{k\in \Gamma}$ and $\lambda_k\in \mathbb{R}^n$ for $k\in \Gamma$.  
Minimizing the Lagrangian with respect to $\bB$ gives the optimum 
\begin{align}
\label{eq:B}
    [\bB( \Lambda )] (x_1,\dots, x_m) =  \exp\left( -\bC(x_1,\dots, x_m)/ \eta\right)  \prod_{k \in \Gamma} \Big( \exp\left(\lambda_k (x_k)/\eta\right) \mu_k(x_k)  \Big),
\end{align}
and plugging this into \eqref{eq:lag_function} yields 
% Thus the Lagrange dual problem is
\begin{align*}
\inf_\bB L(\bB, \Lambda)=L(\bB(\Lambda),\Lambda)= 
% \min_{ \Lambda } \psi( \Lambda) :=
-\eta P( \bB(\Lambda) ) + \sum_{k\in \Gamma} \mu_k^\text{T} \lambda_k.
\end{align*}
Therefore, the dual problem (formulated as a minimization problem) is given by
\begin{align*}
\min_{ \Lambda } \psi( \Lambda) :=
\eta P( \bB(\Lambda) ) - \sum_{k\in \Gamma} \mu_k^\text{T} \lambda_k.
\end{align*}
In each iteration the block coordinate descent algorithm picks some $k \in \Gamma$ and minimizes $\psi(\Lambda)$ over $\lambda_k$, while keeping the other variables fixed. The minimum is achieved when the gradient of $\psi$ with respect to $\lambda_k$ vanishes, i.e., when
\begin{equation}
    e^{\lambda_k(x_k) / \eta}  \mu_k(x_k) \left(\sum_{\bx \setminus x_k} e^{ -\bC(x_1,\dots, x_m)/ \eta }  \prod_{\ell \in \Gamma \setminus k} \Big( e^{\lambda_\ell (x_\ell)/\eta} \mu_\ell(x_\ell)  \Big) \right) - \mu_k(x_k) = 0.
\end{equation}
In the scaled variables $u_k = \exp( \lambda_k / \eta) $ this can be expressed as
\begin{equation}
    u_k \odot \mu_k \odot \left(  P_k( \bB(\Lambda)) ./ \left( u_k \odot \mu_k \right) \right) - \mu_k = 0.
\end{equation}
% \jiao{This better be changed, add t?\begin{equation}
%     u_k^{(t+1)} \odot \mu_k \odot \left(  P_k( \bB(\Lambda^{(t)})) ./ \left( u_k^{(t)} \odot \mu_k \right) \right) - \mu_k = 0.
% \end{equation}}
This yields the Sinkhorn updates \eqref{eq:sinkhorn}.

% \begin{align}
% \label{eq:lamda_t+1}
%     \nabla_{\lambda_k^{(t+1)}} \psi \left(\lambda_1^{(t)},\ldots,\lambda_{k-1}^{(t)}, \lambda_k^{(t+1)},\lambda_{k+1}^{(t)},\ldots, \lambda_{|\Gamma|}^{(t)} \right)= 0
% \end{align}
% by the first order optimal condition. This is equivalent to
% \begin{equation}
%     \mu_k=u_k^{(t+1)} \odot P_k( \bB(\Lambda^{(t)})) ./u_k^{(t)}.
% \end{equation}
% Thus we arrive at the updating rule in \eqref{eq:sinkhorn}.
% Plugging in \eqref{eq:proj_messages} (when $k \in \Gamma$) into \eqref{eq:sinkhorn} yields the $ u_k^{(t+1)}$ updating rule in Algorithm \ref{algo:sinkhorn}.

%\section{Discussion for general graph in Remark \ref{rem:general_graph}}
% \section{The extension to general graphs in Theorem~\ref{thm:general_complexity} } \label{sec:general_graph}
\section{Algorithm for MOT with general graph structure} \label{sec:general_graph}

In order to apply Sinkhorn belief propagation, we first decompose the underlying graph into a tree $\mathcal{T}=(\mathcal{C},\mathcal{E})$ with minimal tree-width.
The cost tensor $\bC$ decouples according to $\mathcal{T}$ into tensors $\bC_c$, for $c\in \mathcal C$, such that
\begin{equation}
    \bC=\sum_{c \in \mathcal{C}} \bC_c(\bx_c).
\end{equation}
The potential tensor $\bK = \exp(- \bC/ \eta)$ %, which is given by
% \begin{equation}
%     \bK(\bx)= \exp(-\bC(\bx)/\eta) = \prod_{(k_1,k_2)\in E} K^{(k_1,k_2)} (x_{k_1}, x_{k_2}),
% \end{equation}
is then factorized, into tensors $\bK_c= \exp(- \bC_c / \eta)$, for $c\in \mathcal{C}$, and can be written as
\begin{equation}
    \bK(\bx) = \prod_{c \in \mathcal{C}} \bK_c(\bx_c).
\end{equation}
% {\red how about $\sum_{\ell \in \mathcal{C}} \bC_\ell(\bx_{c_\ell})$ and $\prod_{\ell \in \mathcal{C}} \bK_\ell(\bx_{c_\ell})$?}
%
To apply Algorithm~\ref{algo:sinkhorn}, the constraints have to be given on the leaf nodes of the tree. 
Thus, we define the junction tree such that the all leaves are clusters containing only one vertex and correspond to the set $\Gamma$. We denote this set of leaf cliques by $\Gamma_\mathcal{C}$.
In particular, note that then $S_{\ell_k k}= x_k$, if $c_k\in \Gamma_\mathcal{C}$, and $c_{\ell_k}$ is its unique neighbour clique.
The Sinkhorn iterations are then of the form \eqref{eq:sinkhorn}, where the projections on the marginals $c_k \in \Gamma_\mathcal{C}$, with neighbour clique 
% \jiao{[if we use $\ell$ for index for $c$, then here it should be $c_{\ell_k} \in \mathcal{C}$?]}
$c_{\ell_k}\in \mathcal{C}$, are computed as
\begin{equation} \label{eq:proj_graph}
    [ P_k( \bB(\Lambda^{(t)}) )](x_k) = u_k^{(t)} (x_k) \mu_k(x_k) m_{\ell_k \to k} ( x_k ).
\end{equation}
Here, the messages between clusters of the junction tree are given by
% \jiao{should be $\bK_{\ell}$ instead of $\bK_{c_\ell}$?} \ih{[I think in the meeting we decided to use the notation $\bK_{c_\ell}$ ]}
\begin{subequations}\label{eq:message_clusters}
\noeqref{eq:message_clusters_a,eq:message_clusters_b}
\begin{align}
   m_{\ell\to k} (S_{\ell k}) & = \sum_{\bx_{c_\ell} \setminus S_{\ell k}} \bK_{c_\ell}(\bx_{c_\ell}) \prod_{j \in N(\ell)\setminus k} m_{j \to \ell} (S_{j \ell}), \quad \text{ if } c_\ell \notin \Gamma_\mathcal{C} \label{eq:message_clusters_a} \\
    m_{\ell \to k} (x_\ell ) & =   u_\ell^{(t)}(x_\ell) \mu_\ell(x_\ell), \qquad \text{ if } c_\ell \in \Gamma_\mathcal{C}. \label{eq:message_clusters_b}
\end{align}
\end{subequations}

It follows that the Sinkhorn iterations \eqref{eq:sinkhorn} with the projections \eqref{eq:proj_graph} read, as before,
\begin{equation}
    u_k^{(t+1)}(x_k) \leftarrow ( m_{\ell_k \to k} (x_k) )^{-1}.
\end{equation}
Algorithm~\ref{algo:sinkhorn} can thus simply be modified to general graphs by replacing the messages \eqref{eq:message_updates} by the messages \eqref{eq:message_clusters}.
This lets us formulate the result in Theorem~\ref{thm:general_complexity}.

 \section{Details on the discussion of results in Section~\ref{sec:discussion}}

 This Section provides details on the discussion of the results.

%\subsection{Discussions of Remark \ref{cor:bound_mot}}
\subsection{Balanced MOT problems}
\label{sec:balanced_tree}

% Here, we discuss several cases where a tree-structured MOT problem satisfies the assumption of Remark~\ref{cor:bound_mot}. We call such a type of problem a balanced MOT problem.
% \begin{defn}
% \label{def:balanced_tree}
%  Given a sequence of tree-structured MOT problems, where the number of nodes go to infinity,
%  we call the sequence of such problems balanced if there is a constant $c$ such that $|\Gamma| \RC \le c\|\bC\|_\infty$.
% \end{defn}

There are many structured MOT problems of interest that are balanced.
In the following we check the condition in Definition~\ref{def:balanced_tree} for a few special cases.

\begin{example}
The Wasserstein barycenter problem discussed in Example~\ref{exa: barycenter} 
% and Remark~\ref{rem:barycenter} 
is balanced. With the barycenter cost tensor $\bC$ defined in \eqref{eq:barycentercost} it holds $\|\bC\|_\infty = \| C\|_\infty$, and thus $|\Gamma| \RC = L \frac{1}{L} \|C\|_\infty = \|\bC\|_\infty$.
\end{example}

\begin{example} \label{ex:barycenter_balanced}
A tree-structured MOT problem where the costs on all edges are equal and symmetric is balanced.
Note that if $C^{(k_1,k_2)}$, for all $(k_1,k_2)\in E$, are equal and symmetric, then $\|\bC\|_\infty= |E| \RC$. Thus, it holds
\begin{equation}
    |\Gamma| \RC = \frac{|\Gamma|}{|E|} \| \bC\|_\infty \leq \| \bC\|_\infty.
\end{equation}
The barycenter case in Example~\ref{exa: barycenter} is a special case of this.
\end{example}

\begin{example}
Consider a tree-structured MOT problem, where the shortest distance between any two leaf nodes is $3$, and the maximum cost entries on the edges connecting to the leaf nodes are of the same order.
Such a problem is balanced.
Let $\|C^{(k,\ell_k)}\|_\infty$ be of the same order for all $k \in \Gamma$, where $\ell_k$ is the neighbour of $k$. Then there is a constant $c$ such that
\begin{equation}
\RC = \max_{k \in \Gamma} \|C^{(k,\ell_k)} \|_\infty \leq \frac{c}{|\Gamma|} \sum_{k \in \Gamma} \|C^{(k,\ell_k)} \|_\infty. 
\end{equation}
If the shortest distance between any two leaf nodes is $3$, there is no node that has two leaf nodes as neighbours. Thus, it holds
\begin{equation}
     \sum_{k \in \Gamma} \|C^{(k,\ell_k)} \|_\infty \leq \|\bC\|_\infty.
\end{equation}
Hence, it follows $|\Gamma| \RC \leq c \|\bC\|_\infty.$

% \jk{
% If no node in the tree is a neighbour of two leaf nodes (same C matrix) \jiao{write in example 3 already}, then 
% \[
% \sum_{k\in \Gamma}R^\Gamma_{C_k}\le \|\bC\|_\infty.
% \]
% Also extends to any case when $0<b\le R^\Gamma_{C_k}\le B$ for all $k\in \Gamma$ then $R_\bC^\Gamma =\mathcal{O}(|\Gamma|^{-1}\|\bC\|_\infty)$. 
% } \jiao{Can this be absorbed in Example 2 or we write another example for this??}

% If no node in the tree is a neighbour of two leaf nodes, then
% \[
% \sum_{k\in \Gamma}R^k_{C}\le \|\bC\|_\infty
% \]
% \jiao{because $R^k_{C}$ are all the same for any $k \in \Gamma$?}
% \ih{minimum distance between two leaves is 2??}
\end{example}

\begin{example}
Consider a tree-structured MOT problem with cost tensor $\bC$.
Let $\tilde \bx = (\tilde x_1,\dots, \tilde x_m)$ be a maximizer of $\bC(\bx)$, that is $\bC(\tilde \bx) = \|\bC\|_\infty$, and assume that
\begin{equation} \label{eq:C_balance}
  \left| \frac{ \max_{k\in \Gamma} \|C^{(k,\ell_k)}\|_{\infty} }  {\min_{ k \in \Gamma }   C^{(k,\ell_k)}( \tilde x_{k_1}, \tilde x_{k_2}) } \right| \leq c
\end{equation}
for some constant $c$. Then the MOT problem is balanced.
To see this note that 
\begin{equation}
     \|\bC\|_\infty =\sum_{(k_1,k_2)\in E} C^{(k_1,k_2)}( \tilde x_{k_1}, \tilde x_{k_2})  \geq |E| \min_{(k_1,k_2)\in E}   C^{(k_1,k_2)}( \tilde x_{k_1}, \tilde x_{k_2}).
\end{equation}
Thus, it follows
\begin{equation}
     |\Gamma| \RC = |\Gamma| \max_{k \in \Gamma} \|C^{(k,\ell_k)}\|_\infty  \leq c |\Gamma| \min_{ k \in \Gamma }   C^{(k,\ell_k)}( \tilde x_{k_1}, \tilde x_{k_2}) 
 \leq c  \frac{|\Gamma|}{|E|}\|\bC\|_\infty 
 \leq  c \|\bC\|_\infty .
\end{equation}

\end{example}

\subsection{Barycenter problem with nonuniform weights} \label{sec:barycenter_nonuniform}
% Consider the barycenter problem introduced in Example~\ref{exa: barycenter}. This problem is a MOT problem \eqref{eq:unregu mot} with underlying graph as illustrated in Figure~\ref{fig:barycenter_graph}. Here, $d(G)=2$, $|\Gamma|=L$, and $m=L+1$. Moreover, by \eqref{eq:barycentercost}, we have $ \RC = \frac{1}{L} \|C\|_\infty $.
% Thus, Algorithm~\ref{algo:complete} is expected to return an $\epsilon-$approximate solution to problem \ref{eq:barycenter_pairwise} in $\mathcal{O}( L n^2 \|C\|_\infty^2 \log(n)\epsilon^{-2} )$.
% This coincides with the best known bound for the barycenter problem \citep{kroshnin2019complexity,lin2020fixedsupport} without acceleration. We also point out that the regularizer used in the Wasserstein barycenter literature is different from ours: one is pairwise regularization and one is regularization over the full tensor $\bB$. For more details on this comparison, see \citet[Section 5]{haasler2020tree}.

We provide a complexity bound for the barycenter problem with nonuniform weights.
Let $w_\ell$ be the weight and $C_\ell$ be the cost matrix for the $\ell$-th term in the barycenter problem \eqref{eq:barycenter_pairwise}.
Then the cost tensor in the corresponding MOT probolem is given by
 $${\bf C}(x_1,\dots,x_L,x_{L+1}) = \sum_{\ell=1}^{L} w_\ell C_\ell(x_{L+1}, x_{\ell}).$$
 With this cost, the bound in Lemma~\ref{lem:psi2psi_star} becomes
 $$\psi(\Lambda)-\psi(\Lambda^*)\le R \sum_k w_k\|P_k({\bf B}(\Lambda))-\mu_k\|_1,\quad \text{ where } R = \max_\ell \|C_\ell\|_\infty.$$ 
 Now, if in Algorithm \ref{algo:sinkhorn} we pick the next update according to the weight $w_1, w_2,\ldots, w_L$ instead of a uniform distribution, then the bound in Lemma~\ref{lemma:incremental psi} becomes 
$$\mathbb{E} \left[ \psi(\Lambda^{(t)}) - \psi(\Lambda^{(t+1)}) \right] \geq  \frac{\eta}{2} \left( e_t\right)^2, \quad \text{ with } e_t= \sum_{k \in \Gamma } w_k\| P_k(\mathbf{B}(\Lambda^{(t)})) - \mu_k \|_1.$$
The bound in Theorem \ref{theo:t bound} then becomes $\mathcal{O}(\frac{R}{\eta \epsilon'})$. Putting everything together, the iteration complexity becomes $\tilde{\mathcal{O}} (\frac{m R^2}{\epsilon^2})$ and the arithmetic complexity becomes $\tilde{\mathcal{O}} (\frac{m n^2 R^2}{\epsilon^2})$, which match the results in \citet{kroshnin2019complexity}. 
% It should also be noted that for many problem formulations it is not obvious how to increase the number of marginals $m$ in a canonical manner since the cost function also needs to be redefined. Most previous research assumes that the infinity norm of the cost tensor $\|{\bf C}\|_\infty$ bounded, as can be obtained by weighting each cost in the barycenter problem by $1/m$. Another natural way of extending the problem is to extend the graph and let the bimarginal cost terms remain bounded. We adopt the latter approach, in this case properties of the graph (degree, treewidth) naturally appears in the computational complexity.

\section{Deferred proofs}

In this section we provide the proofs that are omitted in the main paper.

\subsection{Proof of Lemma \ref{lem:bound_lambda}}

\begin{proof}%[Proof of Lemma \ref{lem:bound_lambda}]

  %We use the bounds $ \exp(-\|C\|_\infty / \eta ) \leq K_{i_0 i_j} \leq 0$ for all $i_0,i_j \in \{1,\dots,n\}$, and $0 \leq (\mu_k)_i \leq 1$ for all $i=1,\dots,n$ and $k=1,\dots,m$.
    Denote $ v_k(x_\ell) = \prod_{j\in N(\ell) \setminus k} m_{j \to \ell} (x_\ell)$, where $\ell \in N(k)$ is the unique neighbour of $k$, since $k$ is a leaf of the tree.
    Assume variable $u_k$ was updated in the previous step of the algorithm. Then it holds
    \begin{equation}
        u_k(x_k) = 1 / m_{\ell \to k} (x_k)  = 1 / \left(
        K^{(k, \ell)}
        v_k \right) .
    \end{equation}
    Thus,
    \begin{equation} \label{eq:lambda_bound_max}
        \max_{x_k} \lambda_k (x_k) \leq - \eta \log\left( e^{ -\|C^{(k,\ell)}\|_\infty/\eta} v_k^T \ones  \right)
        =   \|C^{(k,\ell)}\|_\infty - \eta \log\left( v_k^T \ones \right).
    \end{equation}
    Moreover,
    \begin{equation}  \label{eq:lambda_bound_min}
        \min_{x_k} \lambda_k (x_k) \geq - \eta  \log \left(   v_k^T \ones   \right) .
        % =  - \eta \log\left( v_k^T \ones \right) .
    \end{equation}

    Combining \eqref{eq:lambda_bound_max} and \eqref{eq:lambda_bound_min} it follows
    \begin{equation*}
        \max_{x_k} \lambda_k (x_k)  -\min_{x_k} \lambda_k (x_k) \leq  \|C^{(k,\ell)}\|_\infty .
    \end{equation*}

    Note that the gradient of $\psi(\cdot)$ vanishes in $\Lambda^*$, since it is optimal to \eqref{eq:ot_multi_dual}. Thus, it holds $P_k(\bB(\Lambda^*)) = \mu_k$ for $k=1,\dots,m$ and the bound for $\lambda_k^*$ follows in the same way as before.
   
\end{proof}

\subsection{Proof of Lemma \ref{lem:psi2psi_star}}
\begin{proof}%[Proof of Lemma \ref{lem:psi2psi_star}]

% The result follows as a straightforward generalization of the proof of Lemma 2 in \citet{dvurechensky2018computational} to the multi-marginal setting.
%\ih{ %\st{ $\tilde{\psi}(\Lambda)$ is defined in the proof of Theorem} \ref{theo:t bound}. \st{Plugging in definition of $\psi(\Lambda)$ gives}
Note that
\begin{equation} \label{eq:psi_tilde_proof}
\begin{aligned}
    %\ih{ % \tilde{\psi}(\Lambda) 
     & {\psi}(\Lambda)-{\psi}(\Lambda^*) 
    = \eta P( \bB(\Lambda) ) - \sum_{k\in \Gamma} \mu_k^\text{T} \lambda_k   - \eta P( \bB(\Lambda ^*) ) + \sum_{k\in \Gamma} \mu_k^\text{T} \lambda_k^*  \\
    & \quad =  \eta P( \bB(\Lambda) ) - \sum_{k\in \Gamma}  \lambda_k^\tT P_k( \bB(\Lambda) )   - \eta P( \bB(\Lambda ^*) ) + \sum_{k\in \Gamma}    (\lambda_k^*)^\tT  P_k( \bB(\Lambda) )  + \sum_{k\in \Gamma}   (\lambda_k-\lambda_k^*)^\tT \left( P_k( \bB(\Lambda) )-\mu_k \right).
\end{aligned}
\end{equation}
Consider the convex function of $\widehat \Lambda = \{ \widehat{\lambda}_k \}_{k\in \Gamma}$ given by
% convex because it's negative Lagrangian dual function.
\begin{align*}
 h(\widehat{\Lambda})= \eta P( \bB(\widehat{\Lambda}) ) -  \sum_{k\in \Gamma}  \widehat\lambda_k^\tT P_k( \bB(\Lambda) ).
\end{align*}
% is a convex function with respect to $\widehat{\lambda}_k$ for any $k \in \Gamma$
%and $\nabla_{\widehat{\lambda}_k} h=P_k( \bB(\widehat{\Lambda}) ) - P_k( \bB({\Lambda}) ) =0$ if and only if when $\widehat{\Lambda}= \Lambda$.  So $ \Lambda$ is the minimizer of $h$. Thus $\tilde{\psi(\Lambda)}$ is bounded above by
Note that its gradient vanishes if and only if $\widehat{\Lambda}= \Lambda$, since $\nabla_{\widehat{\lambda}_k} h=P_k( \bB(\widehat{\Lambda}) ) - P_k( \bB({\Lambda}) ) =0$. Thus, $ \Lambda$ is the minimizer of $h$, and it follows with \eqref{eq:psi_tilde_proof} that
\begin{equation}
\label{eq: sum Pk-mu}
\psi(\Lambda)-\psi(\Lambda^*)  \leq \sum_{k\in \Gamma}   (\lambda_k-\lambda_k^*)^\tT \left( P_k( \bB(\Lambda) )-\mu_k \right).
\end{equation}
Define $\bar \lambda_k=  \frac{1}{2}( \max_{x_k} \lambda_k (x_k) + \min_{x_k} \lambda_k (x_k)
)$, and note that $\bar \lambda_k^\tT \left( P_k( \bB(\Lambda) )-\mu_k \right)=0$.
By \holder inequality and Lemma \ref{lem:bound_lambda} , it holds
\begin{align}
\lambda_k^\tT \left[ P_k( \bB(\Lambda) )-\mu_k \right]
= & (\lambda_k - \bar \lambda_k)^\tT \left( P_k( \bB(\Lambda) )-\mu_k \right) \\
\leq & \|\lambda_k - \bar \lambda_k)\|_\infty \left\| P_k( \bB(\Lambda) )-\mu_k \right\|_1 \\
=& \frac{1}{2} \left(\max_{x_k} \lambda_k (x_k) - \min_{x_k} \lambda_k (x_k) \right) \left\| P_k( \bB(\Lambda) )-\mu_k \right\|_1 \\
\leq & \frac{R_C^k}{2}\left\| P_k( \bB(\Lambda) )-\mu_k \right\|_1 . \label{eq:lam_k 1norm}
\end{align}
Similarly, defining $\bar \lambda_k^*= \frac{1}{2} ( \max_{x_k} \lambda_k^* (x_k) + \min_{x_k} \lambda_k^* (x_k)
)$, we derive the bound
\begin{align}
-{\lambda_k^*}^\tT \left( P_k( \bB(\Lambda) )-\mu_k \right)
= & (\bar \lambda_k^*-\lambda_k^* )^\tT \left( P_k( \bB(\Lambda) )-\mu_k \right)
% \leq \|\lambda_k - \bar \lambda_k)\|_\infty \left\| P_k( \bB(\Lambda) )-\mu_k \right\|_1 \\
% =& \frac{\max_{x_k} \lambda_k (x_k) - \min_{x_k} \lambda_k (x_k)
% }{2}\left\| P_k( \bB(\Lambda) )-\mu_k \right\|_1 \\
\leq \frac{R_C^k}{2}\left\| P_k( \bB(\Lambda) )-\mu_k \right\|_1 .
\label{eq:lam star 1norm}
\end{align}
Summing \eqref{eq:lam_k 1norm} and \eqref{eq:lam star 1norm} over $k\in \Gamma$ yields
\begin{align}
\sum_{k\in \Gamma} 
(\lambda_k-\lambda_k^*)^\tT \left( P_k( \bB(\Lambda) )-\mu_k \right)
\leq \sum_{k\in \Gamma}  {R_C^k}\left\| P_k( \bB(\Lambda) )-\mu_k \right\|_1 
\leq   {R_C^\Gamma} \sum_{k\in \Gamma}  \left\| P_k( \bB(\Lambda) )-\mu_k \right\|_1 .
\end{align}
Together with \eqref{eq: sum Pk-mu} this completes the proof.
\end{proof}

\subsection{Proof of Lemma \ref{lemma:incremental psi}}

\begin{proof}%[Proof of Lemma \ref{lemma:incremental psi}]
    Since $P(\bB(\Lambda^t))=1$ for all $ t$ and $u_{k^{(t+1)}}^{(t+1)} ./u_{k^{(t+1)}}^{(t)}= \mu_{k^{(t+1)}}./ P_{k^{(t+1)}} (\bB(\Lambda^{(t)}))$,
    \begin{align*}
        \psi(\Lambda^{(t)}) - \psi(\Lambda^{(t+1)})
        = & \mu_{k^{(t+1)}}^\tT \left(- \lambda_{k^{(t+1)}}^t
        + \lambda_{k^{(t+1)}}^{t+1} \right)                                            \\
        = & \eta \mu_{k^{(t+1)}}^\tT \log \frac{\mu_{k^{(t+1)}}}{P_\ell (\bB(\Lambda^{(t)}))} \\
        = & \eta \KL (\mu_{k^{(t+1)}} \ | \ P_{k^{(t+1)}} (\bB(\Lambda^{(t)}))).
    \end{align*}
    where $\KL$ is the Kullback–Leibler divergence.
    By Pinsker's inequality, we get
    \begin{align} \label{eq:Psi_Pinsker}
        \psi(\Lambda^{(t)}) -\psi(\Lambda^{(t+1)})
        \geq \frac{\eta}{2} \|\mu_{k^{(t+1)}} - P_{k^{(t+1)}} (\bB(\Lambda^{(t)}))  \|_1^2.
    \end{align}
    Since $k^{(t+1)}$ is randomly picked from a uniform distribution over $\Gamma \setminus k^{(t)}$ the expected value of \eqref{eq:Psi_Pinsker} is
    \begin{align*}
        \psi(\Lambda^{(t)}) -\E_{k^{(t+1)}} \left[\psi(\Lambda^{(t+1)})\right]
        \geq \frac{\eta}{2 (|\Gamma|-1)} \sum_{k\in \Gamma} \|\mu_k - P_k (\bB(\Lambda^{(t)}))  \|_1^2.
    \end{align*}
    By Cauchy–Schwarz inequality, it holds
    \begin{align*}
        \E_{k^{(t+1)}}  \left[	\psi(\Lambda^{(t)}) -\psi(\Lambda^{(t+1)})\right]
        \geq \frac{\eta}{2 (|\Gamma|-1)^2} \left( \sum_{k\in \Gamma} \|\mu_k - P_k (\bB(\Lambda^{(t)}))  \|_1 \right)^2.
    \end{align*}
\end{proof}

\subsection{Proof of Theorem~\ref{theo:t bound}}  

We need the following lemma from \citet{altschuler2020random} to connect the per-iteration expected improvement and the number of iterations.
%The proof can be found in \cite[Appendix A.1]{altschuler2020random}.
\begin{lemma}{\cite[Lemma 5.3]{altschuler2020random}} \label{lemma:iter_bound}
    Assume $A>a, h>0$. Let $(Y_t)_{t=1}^\infty$ be a sequence of random variables adapted to a filtration $(\mF_t)_{t=0}^\infty$ such that 
    (i) $Y_0 \leq A$ almost surely, 
    (ii) $0 \leq Y_{t-1}-Y_t \leq 2(A-a)$ almost surely, and
    \begin{align*}
        \text{(iii) } \E \left[Y_t -Y_{t+1}| \mF_t, Y_t \geq a \right] \geq h \quad \forall t =0,1,2,\ldots.
    \end{align*}
    Then the stopping time $s=\min \{t: Y_t \leq a \}$ satisfies 1) the expectation bound $\E[s] \leq \frac{A-a}{h}+1$; and 2)  $\forall \delta \in(0,1/e)$, the probability bound $\P (s \leq \frac{6(A-a)}{h} \log \frac{1}{\delta}) \geq 1-\delta$ holds.
\end{lemma}

\begin{proof}[Proof of Theorem~\ref{theo:t bound}]
    Define the stopping time $\tau:=\min \left\{t: e_t \leq \epsilon'\right\}$.
    Let $\{\mF_t := \sigma \left(\Lambda^{(1)},\ldots,\Lambda^{(t)} \right)\}_t$ be  the natural filtration.
    By Lemma \ref{lem:psi2psi_star} and Lemma \ref{lemma:incremental psi},
    \begin{align*}
        \E \left[ \psi(\Lambda^{(t)}) - \psi(\Lambda^{(t+1)}) |\mF_t, t < \tau \right] \geq  \frac{\eta}{2 |\Gamma|^2} \left(\max\left\{\frac{\psi(\Lambda^{(t)})-\psi(\Lambda^*)}{\RC }, \epsilon' \right\}\right)^2,
    \end{align*}
    For shorthand, denote $\widetilde{\psi}(\Lambda^{(t)})=\psi(\Lambda^{(t)})-\psi(\Lambda^*)$, and let $\tau_1$ be the first iteration when $ \widetilde{\psi}(\Lambda^{(t)}) \leq \RC \epsilon'$ and $\tau_2:=\tau-\tau_1 \geq 0$.
    % Define $Y_t=\tilde{\psi}(\Lambda^{(t)}) $ if $t \leq \tau_1$, and otherwise $\widetilde{\psi}(\Lambda^{(\tau_1)}) - \epsilon'^2 \eta (t -\tau_1)/ 2(|\Gamma|-1)^2$ if $t>\tau_1$. 
    Define 
        \begin{equation}
            Z_t = \begin{cases} \widetilde{\psi}(\Lambda^{(t)}) & \mbox{if}~ t \leq \tau,
            \\ \widetilde{\psi}(\Lambda^{(t)}) - (t -\tau)\frac{\eta(\epsilon')^2}{2|\Gamma|^2} & \mbox{if}~ t>\tau.
            \end{cases}
        \end{equation}
        A direct observation is that $Z_t$ is monotonically decreasing.
        For $t\in [ \tau_1, \tau]$, let  $Y_{t-\tau_1}=Z_{t}$. Then the expected improvement of $Y_t$ per iteration is at least $\frac{\eta(\epsilon')^2}{2|\Gamma|^2} $, that is
    \begin{align*}
        \E \left[ Y_t - Y_{t+1} |\mF_t, Y_t \geq 0 \right] \geq  \frac{\eta(\epsilon')^2}{2|\Gamma|^2}.
    \end{align*}
    With choices $A=\RC \epsilon'$, $a=0$, and $h=\frac{\eta (\epsilon')^2}{2 |\Gamma|^2}$,  clearly $Y_t\le A$ and $0\le Y_t-Y_{t+1}\le 2(A-a)$. Thus, Lemma \ref{lemma:iter_bound} implies 
        \begin{align*}
        \E[\tau_2'] \leq \frac{2|\Gamma|^2 \RC }{\eta\epsilon'} +1 \quad \text{where } \tau_2'= \min \{t: Y_t \leq 0 \}.
    \end{align*}
    Whenever $t \leq \tau,$ we have $\widetilde{\psi}(\Lambda^{(t)} ) \geq 0$ and as such $Z_t \geq 0$. 
    So $\tau:=\min \left\{t: e_t \leq \epsilon'\right\}$ is achieved earlier than $\min\{ t: Z_t \leq 0 \}$ and this implies
    % {\red we need to explain that $Y_t\ge 0$ when $t\le \tau$ ($\widetilde \psi$ is monotone) $Y_t\ge \tilde \E\{\psi(\Lambda^{(t)})\}$ when $t\le \tau$}
    \begin{align}\label{eq:tau_2}
        \tau-\tau_1=\tau_2 \leq \tau_2'= \min \{t: Z_t \leq 0 \} -\tau_1 \quad \Rightarrow \quad  \E[\tau_2] \leq \E[\tau_2'] \leq \frac{2|\Gamma|^2 \RC}{\eta\epsilon'} +1.
    \end{align}
    To bound $\tau_1$, we define $D_0=\RC e_0$ and $D_i:=D_{i-1}/2$ for $i=1,2,\ldots$ until $D_N \leq \RC \epsilon'$. Let $\tau_{1,i}$ be the number of iterations when $D_i \leq \widetilde{\psi}(\Lambda^{(t)})\leq D_{i-1}$.
    Let $t_{1,i}=\min \{t: \widetilde{\psi}(\Lambda^{(t)}) \leq D_{i-1} \}$.
    Consider $A=D_{i-1}$, $a=D_i$,  $h=\frac{\eta}{2 |\Gamma|^2 \RC^2} D_i^2$, and $Y_t=Z_{t+t_{1,i}} $. It holds
    % \jiao{ [redefine $Y_t$ such that $Y_0 \leq D_{i-1}$?]} 
    \begin{align*}
        \E \left[ Y_t - Y_{t+1} |\mF_t, Y_t \geq D_i \right] \geq  \frac{\eta}{2 |\Gamma|^2 \RC^2} \widetilde{\psi}(\Lambda^{(t)})^2 \geq  \frac{\eta}{2 |\Gamma|^2 \RC^2} D_i^2.
    \end{align*}
    In addition $Y_t\le A$ and $0 \le Y_t-Y_{t+1}\le D_{i-1} \le 2 (A-a)$ by the nonnegativity and monotonicity of $Y_t$.
    % {\red In addition, $Y_t\le A$ by monotonicity and $Y_t-Y_{t+1}\le D_{i-1} \le 2 (A-a)$ by the nonnegativity of $Y_t$.} \ih{$Y_0 \leq A$ and $ 0 \leq Y_t-Y_{t+1}\le D_{i-1} \le 2 (A-a)$?}
    %Similarly,
    From Lemma~\ref{lemma:iter_bound} and the definition of the sequence $D_i$ it follows that 
    \begin{align}\label{eq:tau_1i}
        \E[ \tau_{1,i}] \leq  \frac{D_{i-1}- D_i}{ \eta D_i^2} 2|\Gamma|^2 \RC^2 + 1 \leq  \frac{2|\Gamma|^2\RC^2}{\eta D_i}+1.
    \end{align}
    Summing up Equation \eqref{eq:tau_1i} for $i=1,2,\ldots,N$ and Equation \eqref{eq:tau_2} yields
    \begin{align*}
        \E [\tau] \leq \frac{2|\Gamma|^2 \RC}{\eta\epsilon'} +1+ \sum_{i=1}^N \frac{2 |\Gamma|^2\RC^2}{\eta D_i}+N
        \leq \frac{2 |\Gamma|^2 \RC}{\eta\epsilon'} +1+ \frac{4 |\Gamma|^2\RC}{\eta \epsilon'}+ \log_2 \ceil*{\frac{e_0}{\epsilon'}}.
    \end{align*}
    Since 
    $$e_0:=\sum_{k \in \Gamma} \| P_k(\bB(\Lambda^{(t)})) - \mu_k \|_1 \leq \sum_{k \in \Gamma} \| P_k(\bB(\Lambda^{(t)}))\|_1 + \| \mu_k \|_1 =2 |\Gamma|, $$ there is
    $\log_2 \left( \frac{e_0}{\epsilon'} \right) \leq \frac{e_0}{\epsilon'} \leq  \frac{2 |\Gamma|}{\epsilon'}$. 
   And the mild assumption $\eta \leq 0.5 |\Gamma| R_C^\Gamma$ implies that 
    $$1+ \log_2 \ceil* a \leq 1+ \log_2 \ceil* b  \leq \frac{b |\Gamma| R_C^\Gamma}{\eta}, ~~ \forall ~ b \geq a>0, $$ resulting in $1+ \log_2 \ceil*{\frac{e_0}{\epsilon'}} \leq \frac{2 |\Gamma|^2 \RC}{\eta \epsilon'}.$
    % Since $\log_2 \left( \frac{e_0}{\epsilon'} \right) \leq \frac{e_0}{\epsilon'} \leq  \frac{2 |\Gamma|}{\epsilon'} \leq \frac{2 |\Gamma|^2 \RC}{\eta \epsilon'}$, 
    It further follows 
    % \ih{[What happens to the $+1$ above?]}
    \begin{align*}  
    \E [\tau]    \leq    \frac{8 |\Gamma|^2\RC}{\eta \epsilon'}.
    \end{align*} 
    % \jiao{$c_1=8$ }
    % with a universal constant $c_1$.
    Next we prove the high probability bound.
    By Lemma \ref{lemma:iter_bound}, $\forall \delta \in (0,0.5)$,
    \begin{equation}
        \label{eq:tau2_prob}
        \P\left( \tau_2> \frac{12 |\Gamma|^2 \RC}{\eta \epsilon'} \log \frac{2}{\delta}\right) <\frac{\delta}{2}
    \end{equation}
    and with $\delta_i:= {\delta}/{2^{N-i+2}}$ for each $i=1,\ldots,N$,
    \begin{align*}
        \P\left( \tau_{1,i}> \frac{12 |\Gamma|^2 \RC^2}{\eta D_i} \log \frac{1}{\delta_i}\right) <\delta_i.
    \end{align*}
    Given the series summation $\sum_{i=0}^\infty 2^{-i} =\sum_{i=0}^\infty i \cdot 2^{-i} =2$ and the definition of $\delta_i$ and $D_N$, we have
    \begin{align*}
        \sum_{i=1}^N \frac{\log \frac{1}{\delta_i}}{D_i}
        =\frac{1}{D_N}\sum_{i=0}^{N-1} 2^{-i}{\left(\log \frac{4}{\delta}+i\log2 \right)}
        \leq \frac{2}{D_N}{\left(\log \frac{4}{\delta}+\log2 \right)}
        \leq \frac{3}{\RC \epsilon'}{\log \frac{4}{\delta}}.
        % {\red \mbox{or} \frac{3}{R \epsilon'}{\log \frac{4}{\delta}}},
    \end{align*}
    By taking the union over $\tau_{1,i}$ it follows
    % {\red more details}
    \begin{align}\label{eq:tau1_prob}
        \P\left( \tau_{1}> \frac{36 |\Gamma|^2 \RC}{\eta \epsilon'} \log \frac{4}{\delta}\right)
        \leq \sum_{i=1}^N \P\left( \tau_{1,i}> \frac{12 |\Gamma|^2 \RC^2}{\eta D_i} \log \frac{1}{\delta}\right)
        <\frac{\delta}{2}.
    \end{align}
    % \jiao{$$(P(A+B \geq 1) \leq P(A \geq 0.5)+P(B \geq 0.5))$$}
    Taking a union bound over Equation \eqref{eq:tau2_prob} and Equation \eqref{eq:tau1_prob}, we conclude that
    % there exists a universal constant $c$ such that
    % {\red this is slightly different from the statement}
    \begin{equation*}
        \P\left( \tau> \frac{48 |\Gamma|^2 \RC}{\eta \epsilon'} \log \frac{4}{\delta}\right) <{\delta}.
    \end{equation*}
\end{proof}

\subsection{Proof of Lemma~\ref{lem:round bound}}

\begin{proof}%[Proof of Lemma \ref{lem:round bound}]

    Due to the underlying tree structure of the problem it holds
    \begin{align*}
        \langle \bC, \bB\rangle-\langle\bC, \widehat{\bB}\rangle =  \sum_{(k_1,k_2)\in E } \langle C^{(k_1,k_2)}, P_{k_1,k_2}(\bB) - P_{k_1,k_2}(\widehat{\bB}) \rangle. 
    \end{align*}
    By \holder  inequality and \cite[Lemma 7]{altschuler2017near},
    \begin{align}
        \langle \bC, \bB\rangle-\langle\bC, \widehat{\bB}\rangle
        & \leq  \sum_{(k_1,k_2)\in E}  \| C^{(k_1,k_2)}\|_{\infty}  \| P_{k_1,k_2}(\bB) - P_{k_1,k_2}(\widehat{\bB}) \|_1    \\
        & \leq  2 \sum_{k \in \Gamma} \|C^{(k,\ell_k)}\|_{\infty} \|\mu_k-P_k({\bB}) \|_1.
    \end{align}
    %      In the second step it is used that in the case $k_1,k_2 \notin \Gamma$ the bound in \cite[Lemma 7]{altschuler2017near} becomes
    % \begin{equation}
    %  \| P_{k_1,k_2}(\bB) - P_{k_1,k_2}(\widehat{\bB}) \|_1 \leq 2 \left(  \| P_{k_1}(\bB) - P_{k_1}(\bB) \| +  \| P_{k_2}(\bB) - P_{k_2}(\bB) \| \right) = 0.
    % \end{equation}
    In the second step note that $P_{k_1,k_2}(\bB) = P_{k_1,k_2}(\widehat{\bB})$ by construction whenever $k_1,k_2 \notin \Gamma$. Also, note that for $k\in \Gamma$ we can use the bound in \cite[Lemma 7]{altschuler2017near} to get
    \begin{align}
     \| P_{k,\ell_k}(\widehat{\bB}) - P_{k,\ell_k}(\bB) \|_1 
     \leq & 2 \left(  \| P_{k}(\widehat{\bB}) - P_{k}(\bB) \|_1 +  \| P_{\ell_k}(\widehat{\bB}) - P_{\ell}( \bB) \|_1 \right) \\
     =& 2\| P_{k}(\widehat{\bB}) - P_{k}(\bB) \|_1\\
     =& 2\| \mu_k - P_{k}(\bB) \|_1.
    \end{align}

\end{proof}

\subsection{Proof of Lemma~\ref{lem:Bhat_Bstar}}
\begin{proof} %[Proof of Lemma \ref{lem:Bhat_Bstar}]

    Let $\widetilde{ \bY}$ denote the tensor that is returned from Algorithm~\ref{algo:round} with inputs $\bB^*$ and $\{P_k(\widetilde{\bB}) \}_{k\in \Gamma}$.
    Note that $\widetilde{\bB}$ is the optimal solution to 
    \begin{align*}
        \min_{ \bB \in \Pi^m_\Gamma( (P_k({\widetilde{\bB}}) )_{k\in \Gamma} )} \langle \bC, \bB \rangle + \eta H(\bB | \bM ),
    \end{align*}
    which can easily be verified by checking the KKT conditions. Thus, it holds
    \begin{align*}
        \langle \bC, \widetilde{\bB} \rangle + \eta H(\widetilde{\bB} | \bM ) \leq
        \langle \bC, {\widetilde{\bY}} \rangle + \eta H({\widetilde{\bY}} | \bM ).
    \end{align*}
    
    Since $\langle \widetilde{\bB}, \log(\widetilde{\bB}) \rangle \geq -m \log(n)$ and $\langle \widetilde{\bY}, \log(\widetilde{\bY}) \rangle \leq 0$ it follows that
    \begin{equation} \label{eq:round_result}
    \begin{aligned}
        \langle \bC, \widetilde{\bB} \rangle -  \langle \bC, {\widetilde{ \bY}} \rangle 
        & \leq \eta H({\widetilde{\bY}} | \bM ) - \eta H(\widetilde{\bB} | \bM ) \\
        & \leq - \langle \widetilde{\bB}, \log(\widetilde{\bB}) \rangle +  \langle \widetilde \bB-\widetilde \bY, \log \bM\rangle \\
        & \leq \eta m \log(n) + \eta \sum_{k\in\Gamma}\langle P_k(\tilde \bB)-P_k(\tilde \bY),  \log \mu_k\rangle \\
        & = \eta m \log(n).
    \end{aligned}
    \end{equation}
    Lemma \ref{lem:round bound} gives
    \begin{align}
        \langle \bC, {\widetilde{ \bY}} \rangle -\langle \bC, {\bB^*} \rangle  \leq 2 \sum_{k\in \Gamma} \|C^{(k,\ell_k)}\|_{\infty} \|P_k(\widetilde{\bY})-\mu_k \|_1, \label{eq:round1} \\
        \langle \bC, {\widehat{\bB}} \rangle -\langle \bC, {\widetilde{\bB}} \rangle  \leq 2 \sum_{k\in \Gamma} \|C^{(k,\ell_k)}\|_{\infty} \|P_k(\widetilde{\bB})-\mu_k \|_1.\label{eq:round2}
    \end{align}
    Since $	P_k(\widetilde{\bB})=P_k(\widetilde{\bY})$, summing up \eqref{eq:round_result}, \eqref{eq:round1}, and \eqref{eq:round2} concludes the proof.

\end{proof}

% \section{Details on the remarks in Section 4}

% This Section provides details on some of the remarks in Section 4.
\subsection{Proof of Theorem~\ref{thm:general_complexity}}
\label{sec:average_dist}

\begin{proof}
In the case of a general graph, we factorize it according to a junction tree with minimal tree-width and modify the messages in Algorithm~\ref{algo:sinkhorn} to the message passing scheme in \eqref{eq:message_clusters}.
Note that each message update requires at most $\mathcal{O}( n^{w(G)+1})$ operations. In order to perform one iteration of Algorithm~\ref{algo:sinkhorn} on a junction tree, at most $d(\mathcal{T})$ messages have to be updated. Thus, each iteration of Algorithm~\ref{algo:sinkhorn} on a junction tree requires $\mathcal{O}( d(\mathcal{T}) n^{w(G)+1})$ operations. 
The results in Lemma~\ref{lem:bound_lambda}-\ref{lem:Bhat_Bstar} and and Theorem \ref{theo:t bound} can be applied to the junction tree version of the presented methods. In particular, the constant in Lemma~\ref{lem:bound_lambda} is modified to 
$R^k_C = \|\bC_{c_{\ell_k}}( \bx_{c_{\ell_k}})\|_\infty$,
where $c_{\ell_k}$ is the neighbouring clique to $c_k$. 
Letting $\RC = \max_{k\in \Gamma} R^k_C$, the proof follows as the proof of Theorem~\ref{thm:opercomplexity}, where the per-iteration complexity is now $\mathcal{O}( d(\mathcal{T}) n^{w(G)+1})$.
\end{proof}

\subsection{Proof of Theorem~\ref{thm:general_average_dist}}
\label{sec:average_dist}

\begin{proof}
First, note that the expected time of one iteration is $\E[T_t] = \mO(\bar d(G) n^2)$, for all $t$. The expectation of the random variables $T_t$ is thus bounded and equal for all $t$. We also note that
\begin{equation}
    \E[T_t\ones_{\tau \geq t} ]=\E[T_t | {\tau \geq t} ]\P( {\tau \geq t} )=\E[T_t ]\P( {\tau \geq t} ).
\end{equation}
Moreover,
\begin{equation}
    \sum_{t=1}^\infty \E[T_t \ones_{\tau \geq t} ]=\sum_{t=1}^\infty \E[T_t ] \P({\tau \geq t})
= \E[T_t ] \sum_{t=1}^\infty  \P({\tau \geq t})
= \E[T_t ] \E[\tau ] < \infty.
\end{equation}
% $\E[T_t|\tau \geq t]=\E[T_1]$ for any $t$. 
%Additionally, $\E[T_t]$ is changeless for any $t$, so
Thus, by the general Wald's equation \citep{wald1945some}\citep[Lemma 5.6]{altschuler2020random} 
% \jiao{[their proof doesn't rely on independence]}
it follows
\begin{align*}
    \E[T]=\E\left[\sum_{t=1}^\tau T_t \right]
    = \E\left[\tau \right] \E[T_1]
    =\mathcal{O} \left(\frac{ \bar d(G)   mn^2 |\Gamma|^2 (\RC)^2 \log(n)}{\epsilon^2} \right).
\end{align*}
\end{proof}

\section{Additional experiments results}
For brutal force Sinkhorn and Sinkhorn BP, we use the code given by \url{https://github.com/qshzh/cbp} and make necessary modifications, such as random update rules.

We show another set of experiments with a smaller accuracy $\epsilon=0.2$ in this section.
Otherwise the setting is the %All other settings are the same
as in Section \ref{sec:exp}. The curves in Figure \ref{fig:exp_small_eps} are similar to Figure \ref{fig:exp}, which means the run time dependence on $m$ or $n$ is relatively stable no matter how $\epsilon$ varies as long as it is sufficiently large. A very small $\epsilon$ will result in numerical issues; this is a well-known problem for Sinkhorn type algorithms.

\begin{figure}[h]
   \centering
%   \begin{subfigure}{0.5\textwidth}
%       \centering
%       \includegraphics[width=0.48\linewidth]{images/barycenter_fix_m.png}   \includegraphics[width=0.48\linewidth]{images/barycenter_fix_n.png}
%       \caption{Barycenter.}
%      \label{fig:barycenter}
%   \end{subfigure}
%   \begin{subfigure}{0.5\textwidth}
% \centering
% \includegraphics[width=0.48\linewidth]{images/hmm_fix_m.png}
% \includegraphics[width=0.48\linewidth]{images/hmm_fix_n.png} 
% \caption{Hidden Markov Model.}
%  \label{fig:hmm}
%   \end{subfigure}
% \begin{subfigure}{0.5\textwidth}
% \centering
% \includegraphics[width=0.48\linewidth]{images/w2_ls_fix_m.png}
% \includegraphics[width=0.48\linewidth]{images/w2_ls_fix_n.png} 
% \caption{Wasserstein least square.}
%  \label{fig:general}
% \end{subfigure}
   \begin{subfigure}{0.65\textwidth}
       \centering
      \includegraphics[width=0.48\linewidth]{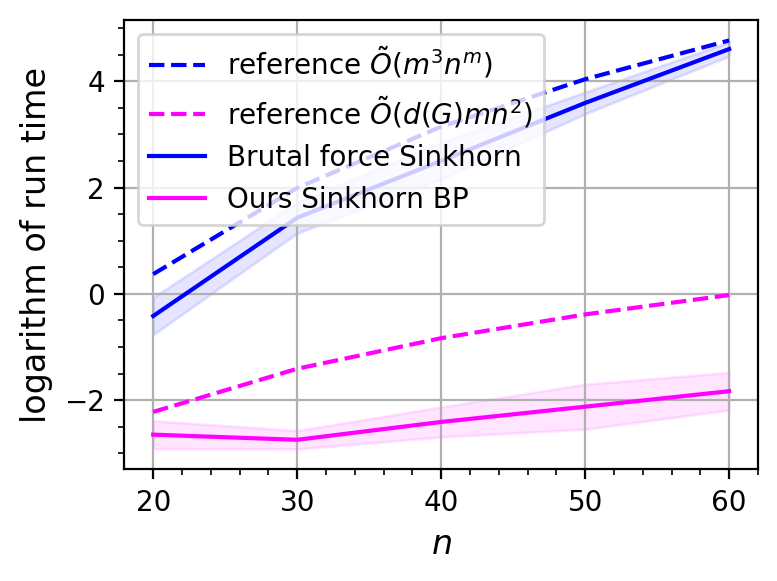}   \includegraphics[width=0.48\linewidth]{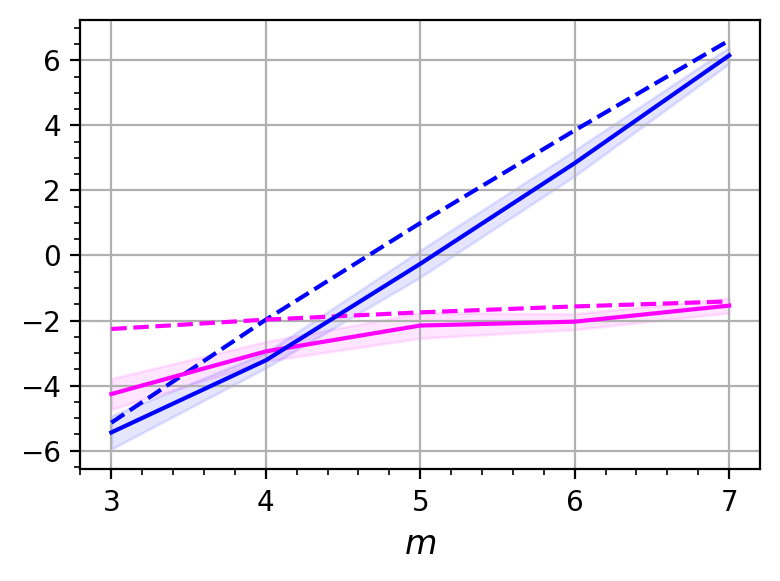}
      \caption{Barycenter.}
   \end{subfigure}
   \begin{subfigure}{0.65\textwidth}
\centering
\includegraphics[width=0.48\linewidth]{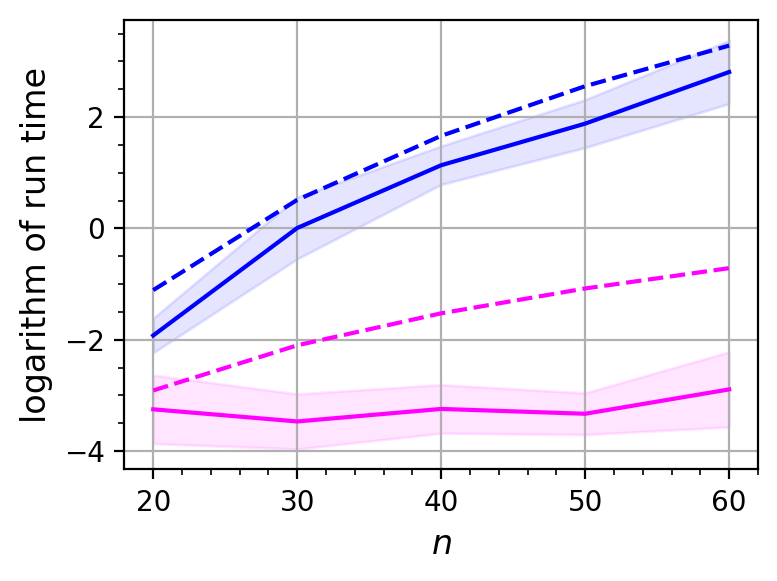}
\includegraphics[width=0.48\linewidth]{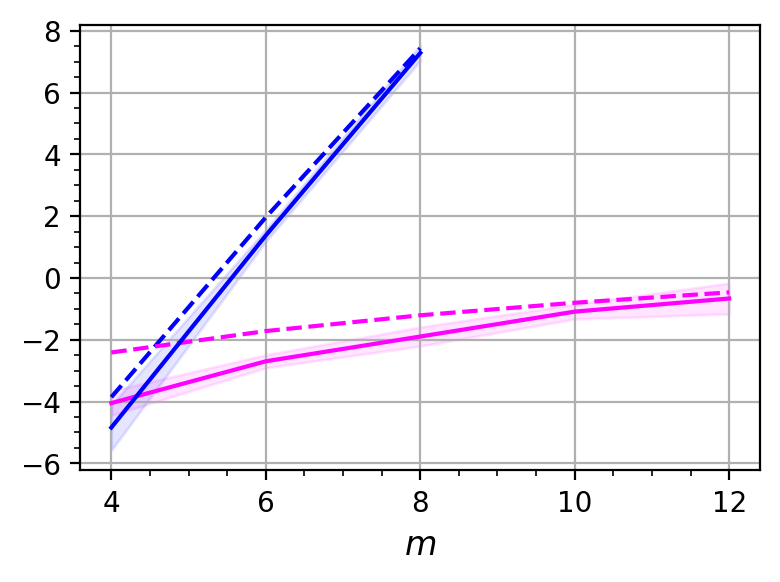} 
\caption{Hidden Markov Model.}
   \end{subfigure}
\begin{subfigure}{0.65\textwidth}
\centering
\includegraphics[width=0.48\linewidth]{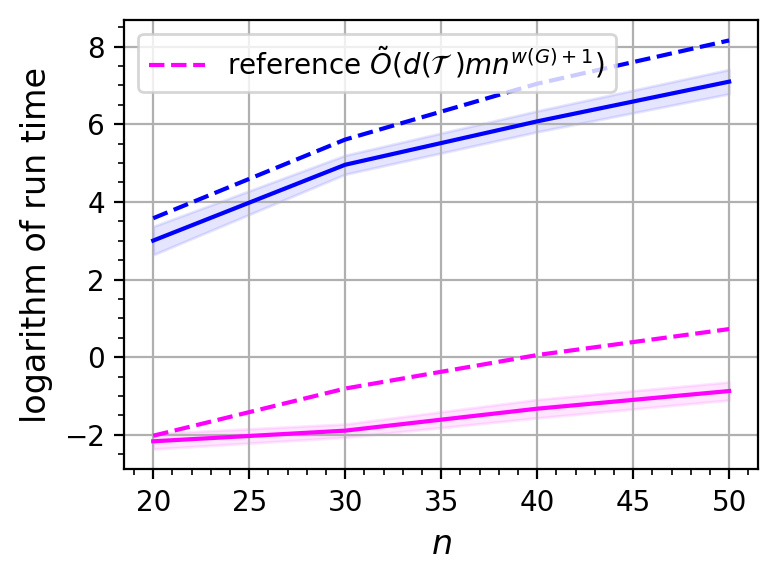}
\includegraphics[width=0.48\linewidth]{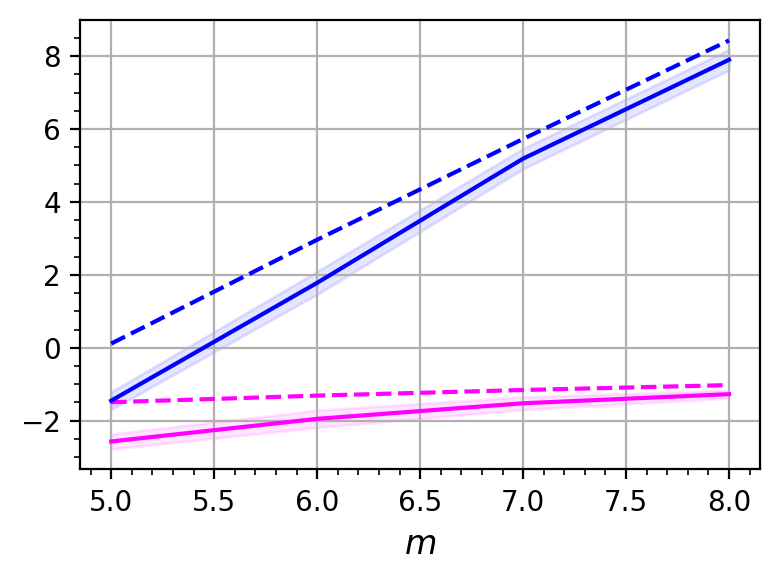} 
\caption{Wasserstein least square.}
\end{subfigure}
\caption{Logarithm of total run time in seconds for brutal force Sinkhorn and Sinkhorn belief propagation. We use a smaller accuracy $\epsilon=0.2$ here.}
%We fix $m$ and study the run time dependence on $n$ in the left column and flip in the right column.}
 \label{fig:exp_small_eps}
\end{figure}

% \jiao{TODO List:

% % - add analysis novelty around contribution? (supple)
% - experiments: add different eps, add details about experiments in main paper.

% [Resubmission statement:]
% If your paper is a resubmission, please briefly summarize the main criticisms by the reviewers and how you addressed them in the resubmission. (This information is only visible to ACs and Program Chairs, but not to reviewers)

% The main criticisms: 1) The paper works out the complexity of the multi-marginal OT problem, where the constraints are on a fixed set of marginals that are connected through a simple graph, which is actually a tree. One reviewer believes this work is more like an extension of two marginal OT, rather than multi-marginal OT.  2)  Experimentally, the authors provide no experiments at all, even to either confirm the tightness of the complexity or provide insights into the convergence.

% How we address them: 1) We focus on the analysis for MOT with general graph structure, which may contain loops and is a much larger class of problems than tree-structured MOT problems. In general graph structured MOT, the problem cannot be directly represented as a sum of bi-marginal OT. In fact, even when the MOT is associated with a tree, the analysis is much more complicated than bi-marginal OT; many key steps in the complexity analysis of bi-marginal OT don't have a counterpart in MOT with tree-structure. 2) We add experiments for tree structured MOT and graph-structured MOT.
% }
\end{document}